\documentclass[final,12pt]{arxiv} 

\usepackage[utf8]{inputenc} 
\usepackage[T1]{fontenc}    
\usepackage{url}            
\usepackage{booktabs}       
\usepackage{amsfonts}       
\usepackage{nicefrac}       
\usepackage{microtype}      
\usepackage{xcolor}         
\usepackage{thm-restate}
\usepackage{afterpage}
\usepackage{enumitem}

\usepackage{tikz-qtree}
\usepackage{tikz}
\usetikzlibrary{shapes,decorations,arrows,calc,arrows.meta,fit,positioning}
\tikzset{
    -Latex,auto,node distance =1 cm and 1 cm,semithick,
    state/.style ={circle, draw, minimum width = 0.7 cm},
    every node/.style={circle, inner sep=0.3mm, minimum size=0.7cm, draw, thick, black, fill=white, text=black},
    every path/.style={thick},
}

\usepackage{graphicx}
\usepackage{mathtools}

\usepackage{wrapfig}
\usepackage{floatrow}
\usepackage[export]{adjustbox}
\usepackage{pifont}
\newcommand{\cmark}{\ding{51}}%
\newcommand{\xmark}{\ding{55}}%

\usepackage{amsmath,amsfonts,amssymb}
\usepackage{bbm}
\usepackage{bm}
\usepackage{xcolor}
\usetikzlibrary{shapes,decorations,arrows,calc,arrows.meta,fit,positioning}
\tikzset{
    -Latex,auto,node distance =1 cm and 1 cm,semithick,
    state/.style ={circle, draw, minimum width = 0.7 cm},
}
\usepackage{algorithm}
\usepackage{algorithmic}

\usepackage{mathdots} 

\newtheorem{defn}[theorem]{Definition}
\newtheorem{thm}{Theorem}
\newtheorem{prop}[theorem]{Proposition}
\newtheorem{cor}[theorem]{Corollary}

\newtheorem{assumption}{Assumption}

\newtheorem*{question*}{Question}
\newtheorem*{answer*}{Answer}
\newtheorem*{solution*}{Solution}
\newtheorem*{nextstep*}{Next Step}
\newtheorem*{issue*}{Issue}

\usepackage{prettyref}
\newcommand{\rref}[2][]{\prettyref{#2}}
\newrefformat{model}{Model\,\ref{#1}}
\newrefformat{listing}{Listing\,\ref{#1}}
\newrefformat{alg}{Algorithm\,\ref{#1}}
\newrefformat{line}{line\,\ref{#1}}
\newrefformat{sec}{Section\,\ref{#1}}
\newrefformat{subsec}{Subsection\,\ref{#1}}
\newrefformat{section}{Section\,\ref{#1}}
\newrefformat{appendix}{Appendix\,\ref{#1}}
\newrefformat{app}{Appendix\,\ref{#1}}
\newrefformat{def}{Definition\,\ref{#1}}
\newrefformat{defn}{Definition\,\ref{#1}}
\newrefformat{thm}{Theorem\,\ref{#1}}
\newrefformat{ax}{\ref{#1}}
\newrefformat{prop}{Prop.\,\ref{#1}}
\newrefformat{lemma}{Lemma\,\ref{#1}}
\newrefformat{cor}{Corollary\,\ref{#1}}
\newrefformat{corollary}{Corollary\,\ref{#1}}
\newrefformat{conj}{Conjecture\,\ref{#1}}
\newrefformat{ex}{Example\,\ref{#1}}
\newrefformat{tab}{Table\,\ref{#1}}
\newrefformat{fig}{Fig.\,\ref{#1}}
\newrefformat{eq}{Equation~(\ref{#1})}
\newrefformat{problem}{Problem\,\ref{#1}}
\newrefformat{assumption}{Assumption\,\ref{#1}}
\newrefformat{conjecture}{Conjecture\,\ref{#1}}
\newrefformat{claim}{Claim\,\ref{#1}}
\newrefformat{remark}{Remark\,\ref{#1}}
\newrefformat{question}{Question\,\ref{#1}}

\newcommand{\kron}{\mathbbm{1}}

\newcommand{\inv}{{-1}}
\newcommand{\DO}{\texttt{do}}

\DeclareMathOperator*{\argmin}{arg\,min}

\DeclareMathOperator{\pa}{pa}

\DeclareMathOperator{\tr}{tr}

\newcommand{\bbE}{\mathbb{E}}
\newcommand{\bbI}{\mathbb{I}}
\newcommand{\bbR}{\mathbb{R}}

\newcommand{\bbP}{\mathbb{P}}

\newcommand{\htheta}{\widehat{\theta}}

\newcommand{\bA}{{\mathbf{A}}}

\newcommand{\bC}{{\mathbf{C}}}
\newcommand{\bD}{{\mathbf{D}}}

\newcommand{\bL}{{\mathbf{L}}}
\newcommand{\bM}{{\mathbf{M}}}

\newcommand{\bP}{{\mathbf{P}}}
\newcommand{\bS}{{\mathbf{S}}}

\newcommand{\bY}{{\mathbf{Y}}}
\newcommand{\bZ}{{\mathbf{Z}}}
\newcommand{\ba}{{\mathbf{a}}}

\newcommand{\bc}{{\mathbf{c}}}

\newcommand{\bq}{{\mathbf{q}}}

\newcommand{\bu}{{\mathbf{u}}}
\newcommand{\bv}{{\mathbf{v}}}
\newcommand{\bw}{{\mathbf{w}}}

\newcommand{\by}{{\mathbf{y}}}
\newcommand{\bz}{{\mathbf{z}}}

\newcommand{\cA}{\mathcal{A}}

\newcommand{\cC}{\mathcal{C}}

\newcommand{\cE}{\mathcal{E}}

\newcommand{\cG}{\mathcal{G}}

\newcommand{\cK}{\mathcal{K}}
\newcommand{\cM}{\mathcal{M}}
\newcommand{\cN}{\mathcal{N}}

\newcommand{\cR}{\mathcal{R}}
\newcommand{\cS}{\mathcal{S}}

\newcommand{\cX}{\mathcal{X}}

\newcommand{\rank}{{\textrm{rank}}}

\newcommand{\onevec}[1]{\textnormal{aug}(#1)}

\title[Causal imputation for counterfactual SCMs]
{Causal Imputation for Counterfactual SCMs: 
\\
Bridging Graphs and Latent Factor Models}

\usepackage{times}

\clearauthor{%
 \Name{{\'A}lvaro Ribot}\thanks{{\'A}. Ribot's contributions to this work were made while visiting MIT IDSS and while affiliated with Centre de Formaci{\'o} Interdisciplin{\`a}ria Superior (CFIS) - Universitat Polit{\`e}cnica de Catalunya (UPC).} \Email{aribotbarrado@g.harvard.edu}\\
 \addr School of Engineering and Applied Sciences, Harvard University, and CFIS, UPC 
 \AND
 \Name{Chandler Squires} \Email{csquires@mit.edu}\\
 \addr Laboratory for Information and Decision Systems, MIT, and Broad Institute of MIT and Harvard%
 \AND
 \Name{Caroline Uhler} \Email{cuhler@mit.edu}\\
 \addr Laboratory for Information and Decision Systems, MIT, and Broad Institute of MIT and Harvard%
}

\begin{document}

\maketitle

\begin{abstract}
    We consider the task of \textit{causal imputation}, where we aim to predict the outcomes of  some set of actions across a wide range of possible contexts.
    %
    As a running example, we consider predicting how different drugs affect cells from different cell types.
    We study the \textit{index-only} setting, where the actions and contexts are categorical variables with a finite number of possible values.
    Even in this simple setting, a practical challenge arises, since often only a small subset of possible action-context pairs have been studied.
    Thus, models must extrapolate to novel action-context pairs, 
    %
    which can be framed as a form of matrix completion with rows indexed by actions, columns indexed by contexts, and matrix entries corresponding to outcomes.
    We introduce a novel SCM-based model class, where the outcome is expressed as a counterfactual, actions are expressed as interventions on an instrumental variable, and contexts are defined based on the initial state of the system.
    We show that, under a linearity assumption, this setup induces a \textit{latent factor model} over the matrix of outcomes, with an additional fixed effect term.
    To perform causal prediction based on this model class, we introduce simple extension to the Synthetic Interventions estimator \citep{agarwalSI}.
    %
    %
    %
    We evaluate several matrix completion approaches on the PRISM drug repurposing dataset, showing that our method outperforms all other considered matrix completion approaches.
\end{abstract}

\vspace{0.2cm}


\addtocontents{toc}{\protect\setcounter{tocdepth}{0}}

\section{Introduction}\label{sec:intro}

A core goal in scientific modeling -- often left implicit  -- is to construct models that accurately predict a system's behavior across a wide range of conditions.
More precisely, we consider the following general \textit{causal prediction} problem: 
\begingroup
\begin{center}
    \textit{
    We know some \emph{context} $C$, i.e., some partial information about the system's state.
    \\
    We are considering whether to perform an \emph{action} $A$ that will affect the system.
    \\
    We wish to predict some \emph{outcome} $Y$, i.e. some feature(s) of the system after performing $A$.
    }
\end{center}
\endgroup

In general, causal prediction may involve high-dimensional contexts, actions, and/or outcomes, such as images or text \citep{chalupka2015visual,castro2020causality,feder2022causal}.
In this work, we study causal prediction in the \textit{index-only} setting\footnote{This setting is similar to the \textit{tabular} setting in reinforcement learning; we choose to use the term \textit{index-only} to emphasize the lack of any additional structure.}, where actions take values in $[m] := \{ 1, \ldots, m \}$ and contexts take values in $[n] := \{ 1, \ldots, n \}$, with the values $i \in [m]$ and $j \in [n]$ carrying no semantic meaning.
In particular, the only information available about an action/context is an index which distinguishes it from other actions/contexts; there is no prior notion of similarity between actions/contexts.
To emphasize our setting, we will write $I_A$ instead of $A$ and $I_C$ instead of $C$.
By restricting our focus to the index-only setting, we aim to maximize the clarity of this work, while building a solid foundation for future works.

As a running example of causal prediction, we consider the problem of \textit{viability prediction} in biology, an essential step in tasks such as drug repurposing \citep{wu2022deep}.
In this problem, we are given a cell type $j \in [n]$ (e.g., skin or lung) and a drug label $i \in [m]$ (e.g. tyloxapol or gepefrine), and we aim to predict what proportion $Y$ of cells will survive if we administer drug $i$ to a large group of type-$j$ cells.
We note that existing databases \citep{drug_database} offer extensive additional information about drugs, which can be used when extending beyond the index-only setting considered here.
Thus, the prediction accuracy in our setting should be seen as a lower bound on the accuracy that could be attained by leveraging this additional information.

\begin{figure}[t]
    \[
    \bY = 
    \underbrace{
    \begin{pmatrix}
        Y_{11} & Y_{12} & ? & \cdots & Y_{1n} \\
        ? & ? & Y_{23} & \cdots & Y_{2n} \\
        Y_{31} & ? & Y_{33} & \cdots & ? \\
        \vdots & \vdots & \vdots & \ddots & \vdots \\
        Y_{m1} & ? & ? & \cdots & ? 
    \end{pmatrix}}_{\text{ \normalsize contexts $j \in [n]$}}
    \begin{rcases*}
    \\
    \\
    \\
    \\
    \\
    \end{rcases*}
    \text{ \normalsize actions $i \in [m]$}
    \]
    \caption{
    \textbf{The causal matrix completion problem}. 
    Each row of $\bY$ corresponds to an action, and each column corresponds to a context. 
    $Y_{ij}$ denotes the outcome after performing action $i$ in context $j$. We represent missing entries with "$?$".
    }
    \label{fig:matrixY}
\end{figure}

To generate causal predictions, we need some model specifying how an action $I_A$ interacts with the context $I_C$ to produce the outcome $Y$.
In statistics and machine learning, a model is obtained via \textit{data-driven} approaches, which consider a fixed \textit{model class} $\Theta$ and use data to select a model $\htheta$ from $\Theta$ (called \textit{learning} or \textit{model fitting}).
In this work, we consider \textit{supervised} learning, where the available data consists of samples of $(I_A, I_C, Y)$.
%
For large $m$ and $n$, the available data often contains samples from only a small subset $\Omega$ of all possible $m \cdot n$ possible action-context pairs.
For example, given $n \approx 100$ cell types and $m \approx 10,000$ drugs, we have $m \cdot n \approx 1$ million, but our dataset might cover only 5-10\% of these pairs.
By arranging the available data into a matrix with rows indexed by actions and columns indexed by contexts, we obtain a partially-observed matrix $\bY$, as in \rref{fig:matrixY}.
In these situations, causal prediction requires \textit{extrapolating} from $\Omega$ to the entire space $[m] \times [n]$, a problem known as \textit{causal imputation} \citep{squires2022causal}.

For extrapolation to be feasible, the model class $\Theta$ must encode some inductive biases governing the interplay between actions, contexts, and outcomes.
%
%
The model class must also be reasonably well-specified, i.e., there must exist some $\theta^* \in \Theta$ which accurately describes the relationship between $Y$, $I_A$, and $I_C$, at least to a good approximation.
Given the importance of the model class $\Theta$, this work aims to advance our understanding of the relationship between two model classes that are commonly used for causal prediction.
In particular, we study the class of \textit{latent factors models} (LFMs) and the class of \textit{structural causal models} (SCMs).

LFMs, also known as \textit{interactive fixed effects models}, are widely used in econometrics and recommendation systems \citep{Athey_2021,koren2009}.
In an LFM, each action $i \in [m]$ is associated with an unknown vector $\bu_i \in \bbR^r$, and each context $j \in [n]$ is associated with an unknown vector $\bv_j \in \bbR^r$.
LFMs assume $Y_{ij} = \langle \bu_i, \bv_j \rangle + \varepsilon_{ij}$ for some mean-zero $\varepsilon_{ij}$, and can be seen as a simple form of SCM, see \rref{fig:causalfactor}.
We can also extend these models to include terms for a \textit{fixed action effect} and/or a \textit{fixed context effect} by letting $Y_{ij} = \langle \bv_i, \bu_j \rangle + \alpha_i + \beta_j + \varepsilon_{ij}$ for some $\alpha_i, \beta_j \in \bbR$.
LFMs are related to low-rank factorizations: letting $\bL \in \bbR^{m \times n}$ with $\bL_{ij} = \bbE[Y_{ij}]$, the definition of the LFM (with no fixed effects) implies that $\rank(\bL) \leq r$.
%
%

%
%
%
%
LFMs are often assumed as a model class without any ``deeper'' justification.
One intuitive way to motivate LFMs is to show that they arise from other (potentially more transparent) modeling assumptions, as in \cite{udell2017big}.
Relating LFMs to other model classes is important for several reasons: such connections offer insights that could legitimize trust in the model's prediction, and can serve as a starting point from which to develop more general model classes.
Thus, a primary aim of this work is to offer a new justification for LFMs, starting from the assumption that the system's state obeys a structural causal model (SCM), with $Y_{ij}$ defined as a \textit{counterfactual} outcome under action $i$ when the system is in context $j$. 
That is, in this paper, we work with quantities that are defined in terms of a counterfactual distribution, and which cannot be defined only using interventional distributions.
Such quantities are commonly seen, for example, in the literature on mediation analysis \citep{malinsky2019potential}.

Finally, note that there is a substantial difference between the roles that columns (contexts) and rows (actions) are playing. Inspired by this distinction, let us call a matrix completion approach \emph{symmetric} when, performed on $\bY^{\top}$, we obtain the same predictions. We want to use a non-symmetric approach that is compatible with our causal viewpoint.


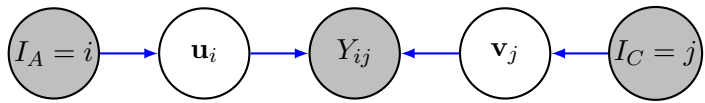
\begin{figure}[t]
    \centering
    \begin{tikzpicture}[align=center]
        \node[fill=lightgray,minimum size = 1.2cm] (0) at (0,0) {$I_A = i$};
        \node[minimum size = 1.2cm] (1) at (2,0) {$\bu_i$};
        \node[fill=lightgray, minimum size = 1.2cm] (2) at (4,0) {$Y_{ij}$};
        \node[minimum size = 1.2cm] (3) at (6,0) {$\bv_j$};
        \node[fill=lightgray, minimum size = 1.2cm] (4) at (8,0) {$I_C = j$};
        \draw[blue] [-latex] (0) edge (1);
        \draw[blue] [-latex] (1) edge (2);
        \draw[blue] [-latex] (3) edge (2);
        \draw[blue] [-latex] (4) edge (3);
\end{tikzpicture}
    \caption{\textbf{The latent factor model (LFM) written as a simple structural causal model (SCM).} 
    LFMs assume that $Y_{ij} = \langle \bu_i, \bv_i \rangle + \varepsilon_{ij}$ for $\bu_i,\bv_i \in \bbR^r$.
    This generative process can be viewed as an SCM over 3 observed (shaded) and 2 latent (unshaded) variables.
    Here, $I_A$ and $I_C$ represent action and context indices, respectively (note that they are independent).
    %
    %
    }
    \label{fig:causalfactor}
    
\end{figure}

\paragraph{Organization of the paper and contributions.} 
After an overview of related work in \rref{sec:relatedwork}, we review SCMs and formally define our model class in \rref{sec:setup}.
We specialize to linear models in \rref{sec:theory}, where we establish the following results:
\begin{itemize}[noitemsep,topsep=0pt, leftmargin=5.5mm]
    \item We generalize the model class from \cite{squires2022causal} to allow for the context $I_C$ to be defined in terms of the system's state, which was not previously allowed.
    In \rref{thm:linear}, we show that our new SCM-based model class implies an LFM with fixed action effects.
    
    \item We propose a slight modification to the \textit{Synthetic Interventions} estimator of \cite{agarwal2021causal} that accounts for fixed action effects.

    \item We give additional conditions on contexts and actions, which imply further structure on $\bL$; see \rref{prop:lineargaussian} and \rref{cor:fe}.
    We show how this additional structure can be leveraged to improve estimation and hypothesis testing.
\end{itemize}
Finally, in \rref{sec:experiments}, we compare the performance of several causal prediction approaches on the PRISM drug repurposing dataset, showing that our method outperforms alternative approaches.

\section{Related Work} \label{sec:relatedwork}

Since this paper bridges between latent factor models and structural causal models, the potential scope for related work is vast.
We limit our discussion to works that consider the index-only setting.
These works are complementary to recent works which give extrapolation guarantees for more general/structured action spaces, such as \cite{saengkyongam2023identifying} and \citet{agarwal2023synthetic}.
A summary of the algorithms presented can be found in Table \ref{table}. We provide the formulae for these algorithms in \rref{appendix:formulas}.
Finally, recall that $\bL \in \bbR^{m \times n}$ with $L_{ij} = \bbE[Y_{ij}]$ and that $\Omega$ denotes the set of observed action-context pairs.

\paragraph{Local Approaches to Matrix Completion.} 
%
We call a causal prediction approach \textit{local} if it predicts individual entries $L_{ij}$, rather than predicting the entire matrix $\bL$ in one go.
Many local approaches predict $L_{ij}$ via some weighted average $\sum_{(k,\ell) \in \Omega} w_{k\ell} Y_{k\ell}$, with different choices for $w_{k\ell}$ giving rise to different methods.
For example, in \textit{Collaborative Filtering} (CF) \citep{schafer2007collaborative,amazon} and nearest-neighbors approaches \citep{dwivedi2022doubly}, the weight $w_{k\ell}$ depends on some measure of data-dependent similarity (e.g. cosine similarity) between actions $i$, $k$ and contexts $j, \ell$.
%

Very simple examples of local, weighting-based estimators include the \textit{mean-over-contexts} estimator (which has $w_{i'j'} \propto \kron_{i' = i}$ with the weights summing to one), and the similarly-defined \textit{mean-over-actions} estimator.
The estimators are appropriate for models with \textit{only} fixed action effects and fixed context effects, respectively.
%
%
To handle a two-way fixed effect model (one with both fixed action and context effects), the mean-over-actions and mean-over-contexts estimators can be combined into the \textit{Fixed Effects} (FE) estimator.
More precisely, this estimator corresponds to taking an average of the fixed action effect estimators presented in \cite{squires2022causal}.

%

Connections between causal prediction and matrix completion are explored in a number of works.
In particular, the widely-used \textit{Synthetic Controls} method predicts outcomes for treated units under the counterfactual setting where they received no treatment \citep{abadie, doudchenko2016balancing, bai2019mccounterfactuals}.
\citet{agarwalSI} generalized this idea to predict counterfactuals under treatment with the \textit{Synthetic Interventions} (SI) estimator, which has also been connected to causal matrix/tensor completion \citep{agarwal2021causal, squires2022causal}.
%


\paragraph{Global Approaches to Matrix Completion.} Other approaches, such as \textit{nuclear norm minimization} (NNM) \citep{candes2009matrix}, predict the entire matrix $\bL$ in one go, often by casting matrix completion as an optimization problem and developing fast optimizers \citep{cai2008singular, candes2008exact, mazumder10a}. 
In fact, there are also global approaches that take advantage of latent factor models \citep{hastie2014matrix, jain2012lowrank}.
There is a vast literature studying the optimality properties of NNM and its variations \citep{candestao, Recht_2010, zhang2015exact}. 
However, these results rely on the assumption of uniformly at random observations, which is usually not satisfied in real datasets, especially for biological applications (see \rref{sec:experiments}). 
Thus, we might prefer to use \textit{local approaches} when the missingness pattern is far from uniform.
%

%



Among global approaches, the one most related to the present work is an extension of NNM by \citet{Athey_2021}.
This approach excludes fixed effects from regularization, hence we call it \textit{NNM-FE}.
%
%
Roughly speaking, if we let $\hat{\bY}^{\text{FE}}$ denote the fixed effect estimator for $\bY$, NNM-FE is similar to using NNM on the matrix $\bY - \hat{\bY}^{\text{FE}}$.
In \rref{sec:theory}, we use the same idea to improve the SI estimator when the model class includes fixed effects.

\paragraph{Causal Prediction in Biological Applications.} We use the DepMap PRISM dataset \citep{depmap}, which measures the viability score of drug and cell line combinations.
\citet{radhakrishnan2022transfer} used additional information from the Connectivity Map (CMAP) dataset \citep{depmap} to predict the viability scores from DepMap, going beyond the index-only setting considered here.
In the index-only setting, \citet{squires2022causal} used Synthetic Interventions for the CMAP dataset while \citep{hodos2018cell} used a nearest-neighbor approach.



\begin{table}[t!]
\vspace{0.2cm}
\begin{center}
\begin{tabular}{r||c|c|c}
\multicolumn{1}{r||}{Algorithm} & 
\textit{Model Class}
& 
\textit{Symmetric }
& 
\textit{Dependence on $|\Omega|$}
\\ 
\hline
Mean-Over-Contexts  & Fixed Action Effect Model & \xmark & Low  \\
Mean-Over-Actions  & Fixed Context Effect Model & \xmark & Low  \\
Fixed Effects (FE)  & Two-way Fixed Effect Model & \cmark  & Low  \\
Collaborative Filtering (CF)        & Mixture Model   & \xmark        & Medium \\
\hline
Synthetic Interventions (SI)    & Latent Factor Model  & \cmark *      & Medium \\
SI-mean-contexts & Latent Factor Model   & \xmark       & Low    \\
SI-FE     & Latent Factor Model  & \cmark * & Low                          
\\
Nuclear Norm Min. (NNM)        & Low Rank   & \cmark  & High \\
NNM-FE  & Low Rank & \cmark       & Low 
\end{tabular}
\end{center}
\vspace{-0.5cm}
\caption{\textbf{Summary of the algorithms}. 
\textit{Model Class} indicates that the algorithm is consistent when the true data-generating process belongs to that class, e.g., Synthetic Interventions is consistent under a factor model.
\textit{Dependence on $|\Omega|$} (the number of observed entries) is a qualitative judgement based on the performance results shown in \rref{sec:experiments}.
\\ {\footnotesize *SI is symmetric for matrices but it is non-symmetric for third-order tensors.}}
\label{table}
\end{table}
\section{Background and Setup} \label{sec:setup}

We begin by reviewing relevant background on structural causal models \citep{peters2017elements}, and then use the concepts to define the model class that we consider for causal prediction.

\subsection{Background}
\begin{defn}[Structural Causal Model (SCM)]
A \emph{structural causal model (SCM)} is defined by a tuple $\cS = (\bS, \bbP^{\cE})$, where $\bS = (S_1, \dots, S_q)$ is an indexed set of $q$ \emph{causal mechanisms}
\[
S_k : \quad Z_k = f_k(\pa(Z_k), \varepsilon_k), \quad k = 1, \dots, q.
\]
Here, $\pa(Z_k) \subseteq \{Z_1, \dots, Z_q\} \setminus \{Z_k\}$ are called the \emph{parents} of $Z_k$, and $\bbP^\cE = \bbP^{\varepsilon_1, \dots, \varepsilon_q}$ is the joint distribution of the \emph{exogenous noise variables} $\varepsilon_1, \ldots, \varepsilon_q$.
The \emph{causal graph} $\cG$ of an SCM has a directed edge $Z_l \to Z_k$ for all $k$ and for all $Z_l \in \pa(Z_k)$.
We assume that $\cG$ is acyclic, i.e., it is a DAG.
\end{defn}

Unless otherwise noted, we assume that $\bbP^\cE$ is a product distribution, i.e., that the noise variables $\varepsilon_1, \ldots, \varepsilon_q$ are jointly independent.
An SCM is called \textit{Gaussian} if $\varepsilon_k \sim \cN(0, \sigma_k^2)$ with $\sigma_k > 0 \; \forall k$. 
It is called \textit{linear} if all causal mechanisms $f_k$ are linear. 
In particular, in a linear SCM, the causal mechanisms are defined by some parameters $B_{l,k} \in \bbR$ and they can be written as
\begin{equation}\label{eq:linear}
    Z_k = \sum_{Z_l \in \pa(Z_k)} B_{l,k}Z_l + \varepsilon_k 
    \quad\quad \textnormal{for~all~}k = 1, \dots, q    
\end{equation}



%
We now define interventions and counterfactuals, mostly following the notation of \cite{peters2017elements}.
Fix an SCM $\cS = (\bS, \bbP^{\cE})$ over nodes $\bZ$.
A \textit{(soft) intervention}\footnote{The terminology for different classes of interventions is fairly inconsistent. Soft interventions have many other names, e.g., \textit{parametric interventions} \citep{eberhardt2007interventions} or \textit{mechanism shifts} \citep{tian2001causal}. Our definition for soft interventions contains \textit{perfect (hard) interventions} and \textit{do-interventions} as special cases, see \citep{squires2022review} for additional terms for these classes.} 
$I$ is defined by a set $T(I) \subseteq \bZ$ of \textit{intervention targets}, and an indexed set $\{ f_k^I \}_{k \in T(I)}$ of \textit{interventional causal mechanisms}, where $f_k$ is generally allowed to be a function of $\pa(Z_k)$ and $\varepsilon_k$.
Given an intervention $I$, we define the \textit{interventional SCM} as $\cS_I = (\bS^I, \bbP^\cE)$, where $(\bS^I)_k = \bS_k$ if $k \in T(I)$ and $(\bS^I)_k = \bS_k$ otherwise.

On the other hand, letting  $\bC \subseteq \bZ$ and conditioning on $\bC = \bc$, we define the \textit{counterfactual SCM} as $\cS_{\bC = \bc} = (\bS, \bbP^{\cE \mid \bC = \bc})$.
To combine counterfactuals and interventions (in that order), we define $\cS_{[\bC = \bc, I]} = (\bS^I, \bbP^{\cE \mid \bC = \bc})$.
This SCM entails a new joint distribution over $\bZ$, which we denote by $\bbP(\bZ \mid [\bC = \bc, I])$, or in the special case of a do-intervention setting $\bA \subset \bZ$ to $\ba$, $\bbP(\bZ \mid [\bC = \bc, \DO(\bA = \ba)])$.

\subsection{An SCM-based Model Class for Causal Prediction}\label{sec:prob_setup}
We give general results for the setting where we observe a vector-valued outcome $\bY_{ij} \in \bbR^p$ for each action-context pair $(i, j)$.
This gives rise to a partially observed 3-order tensor $\bY \in \bbR^{m \times n \times p}$ where the  rows (first index) correspond to actions and the columns (second index) correspond to contexts.
Let $\Omega$ denote the set of indices $(i,j)$ for which we have observed $\bY_{ij} \in \bbR^p$.
Our ultimate goal is to impute the missing entries of $\bbE[\bY]$.
In the special case $p = 1$, $\bY$ reduces to an $m \times n$ matrix, and tensor completion reduces to matrix completion.
%

We assume that there is an underlying SCM over some $\bZ = (Z_1, \dots, Z_q)$ that defines our system (each possible context).
Let $I_A$ and $I_C$ be index variables that define the action and context, respectively. 
We add them to our SCM as follows.
First, we observe $I_C$ before applying any action. 
As $I_C$ defines a context, it is defined as a function of $\bZ$, so it is downstream from every node. 
Then, conditioning on $I_C = j$ corresponds to conditioning on some $\bZ_{\cC_j}$ where $\cC_j \subset [q] := \{ 1, \dots, q\}$.

On the other hand, the index $I_A$ can be thought of as an \textit{instrumental variable} \citep{newey2003instrumental} or a \textit{regime indicator} \citep{dawid2021decision}, which encodes the (unknown) intervention that each action applies.
In particular, using action $i$ corresponds to setting $I_A = i$, which induces an intervention on some set of variables $\bZ_{\cA_i}$, $\cA_i \subset [q]$.
Together, the intervention on $I_A$ and conditioning on $I_C$ gives rise a new SCM, and defines a new set of variables $\bZ(i)$ that represent the counterfactual state of the system.\footnote{This notation is inspired by Single World Intervention Graphs (SWIGs) \citep{richardson2013single}. Indeed, we would obtain a SWIG after performing $\DO$-interventions. However, we also consider soft interventions. We could write $I_A = \varnothing$ for defining the control state, i.e. no action. In that case $\bZ(\varnothing) = \bZ$.}
Finally, we only observe a subset of $\bZ(i)$, which we denote with some indices $\cX \subset [q]$, $|\cX| = p$, and the outcome $\bY_{ij}$ is drawn from the distribution $\bbP(\bZ_\cX(i) \mid I_C = j)$.
This definition of our model class is summarized in \rref{fig:general}.
Altogether, the expected value of the $(i,j)$-th entry of $\bL$ is 
\begin{equation}\label{eq:counterfactual}
    \bbE(\bY_{ij}) = \bbE\left(\bZ_\cX(i) \mid I_C = j \right)
\end{equation}
\begin{figure}[t]
    \centering
    \begin{center}
    \includegraphics[width=0.95\textwidth]{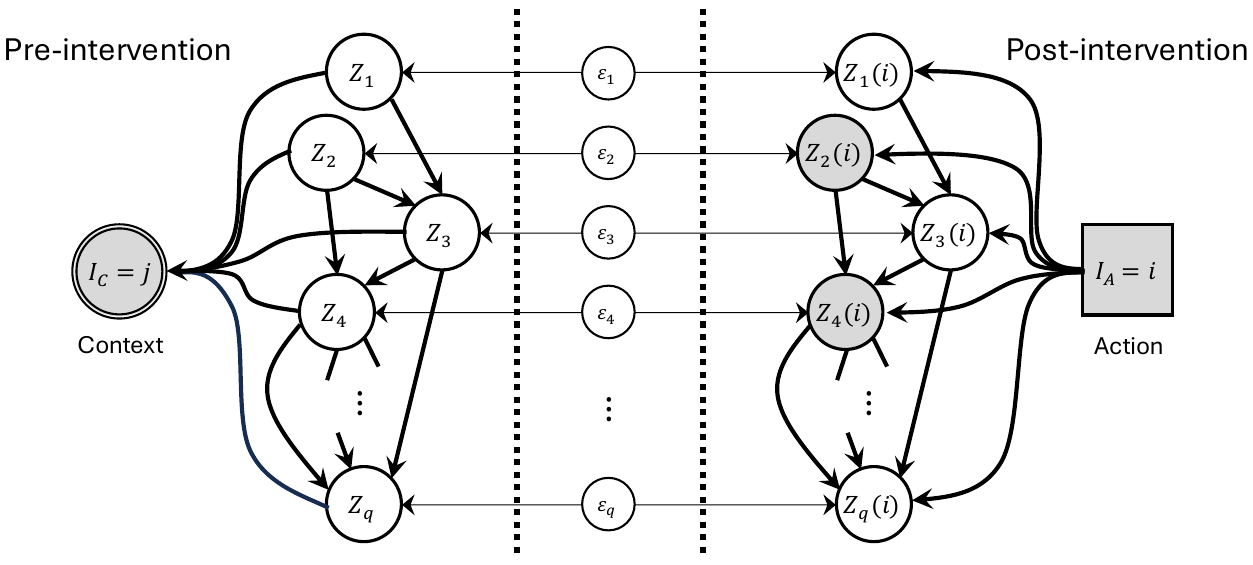}
    \end{center}
     \caption{
     \textbf{DAG defining our model class.} 
     Shaded nodes are observed, while unshaded nodes are unobserved.
     Data is generated by conditioning on the context index $I_C = j$ (indicated by the double-edge for the node $I_C$) followed by intervening to set the action index $I_A = i$ (indicated by the square node $I_A$).
     The context may potentially depend on any subset of $\bZ$ and the action may potentially affect any subset of $\bZ$.
     The exogenous noise terms $\varepsilon_1, \varepsilon_2, \ldots, \varepsilon_q$ are shared between the pre-interventional and post-interventional SCMs.
     Here, the observed outcome is $\bY_{ij} \sim \bbP ( \bZ_\cX(i) \mid I_C = j)$ for $\cX = \{ 2, 4 \}$.
     }
     \label{fig:general}
\end{figure}
We clarify our setup with a simple example which uses only $\DO$-interventions.
\begin{example}
    Let $Z_1 \to Z_2 \to Z_3$ follow an SCM $\cS = (\bS, \bbP^{\cE})$.
    Assume that context is defined solely based on the value of the marker gene $Z_3$, e.g., $I_C = 1$ if and only if $Z_3 = c_1$ and $I_C = 2$ if and only if $Z_3 = c_2$. This gives us a new distribution on the noise, $\bbP^{\cE \mid Z_3 = c_j}$ for each possible $c_j$.
    Assume also that each action corresponds to a $\DO$-intervention on the gene $Z_1$, e.g. if $I_A = 1$, then we have performed the intervention $\DO(Z_1 = a_1)$, if $I_A = 2$, then we have performed the intervention $\DO(Z_1 = a_2)$.
    So for $I_A = i$ we have a new set of structural equations $S^{\DO(Z_1 = a_i)}$. 
    Therefore, the SCM after conditioning and intervening is $\cS_{[Z_3 = c_j, \DO(Z_1 = a_i)]} = (\bS^{\DO(Z_1 = a_i)}), \bbP^{\cE \mid Z_3 = c_j})$.
    Finally, assume that we use an assay which measures only the gene $Z_2$.
    Then our matrix satisfies
    \[
    \bbE ( \bY_{ij} ) = \bbE (Z_2 \mid [ Z_3 = c_j, \DO(Z_1 = a_i) ] )
    \]
    where we use the notation $[Z_3 = c_j, \DO(Z_1 = a_i)]$ to emphasize that the order of the terms is important: the intervention is considered in the model obtained \textit{after} conditioning. Using our notation, shown in \rref{eq:counterfactual}, we would write $\bbE ( \bY_{ij} ) = \bbE (Z_2 (a_i )\mid Z_3 = c_j)$.
\end{example}

\subsection{Synthetic Interventions}

The Synthetic Interventions (SI) estimator \citep{agarwalSI} is a local causal prediction approach for latent factor models.
To predict $\bY_{ij}$ for $(i, j) \not\in \Omega$, SI follows the following procedure: (1) take the set of \textit{columns} $\cC(i)$ for which we have observed row $i$, (2) take all \textit{rows} for which we have observed columns $\cC(i) \cup \{j\}$, (3) fit a linear regression on the available data defined by these subsets, and (4) use the linear regression coefficient to predict $\bY_{ij}$ as a linear combination of $\{ \bY_{i\ell} \}_{\ell \in \cC(i)}$.
Note that, analogously, we could use SI regressing along columns.
See \rref{appendix:formulas} for more details.


One of the main advantages of using SI is that, in contrast of NNM, it does not require significant assumptions on the missingness structure of the data. Instead, to prove theoretical guarantees, e.g. finite-sample consistency \citep{agarwalSI}, the data should be generated from a latent factor model, and also satisfy a \textit{linear span inclusion} assumption, namely that  $\bv_j \in \text{span}(\bv_k : k \in \cC(i))$. Intuitively, if $\cC(i)$ is sufficiently large in relation to the rank of our factor model, SI provides a consistent estimator.
\section{Theoretical Results} \label{sec:theory}

In this section, we establish expressions for the expected value of the entries of our tensor/matrix $\bY$. 
To connect these expressions to specific tensor completion approaches, we will call an estimator \textit{consistent} for an expression if it returns $\bbE[\bY_{ij}]$ when given as input the expected values for each observed entry, i.e. the values $\bbE[\bY_{uv}]$ for $(u, v) \in \Omega$.
This corresponds to taking a limit, when each observed entry is the average of $K$ independent samples of $\bY_{uv} \sim \bbP_{I_C = v ; \DO(I_A = u)}$, and we let $K \to \infty$. 
To simplify notation, for any $\bv \in \bbR^d$, we define $\onevec{\bv} = (1, v_1, \ldots, v_d)^\top \in \bbR^{d+1}$ as the vector given by prepending a 1 to $\bv$. 
The proofs of the following results can be found in \rref{appendix:proofs}.

\begin{assumption} \label{assumption:linear}
    The entries of $\bY$ come from a linear SCM over a set of variables $\bZ = (Z_1, \dots, Z_q)$ as in \rref{eq:linear}. That is, the $(i, j)-$th entry corresponds to the counterfactual when conditioning on $I_C = j$ and then intervening to set $I_A = i$, as in \rref{eq:counterfactual}.
\end{assumption}

\begin{thm}\label{thm:linear}
Under \rref{assumption:linear}, we have
\[
\bbE(\bY_{ij}) = \bbE\left(\bZ_\cX(i) \mid I_C = j \right) = U_i \bv_i + U'_i \bw_j.
\]
for some $U_i, U'_i \in \bbR^{|\cX| \times |Z|}, \bv_i \in \bbR^{|Z|}$ depending on the action index $i$, and some $\bw_j \in \bbR^{|Z|}$ depending on the context index $j$.
\end{thm}
\begin{proof}[Sketch]
For a linear SCM, we have $\bZ = (I - B)^\inv \cE$ with $B \in \mathbb{R}^{q \times q}$ and $\mathcal{E} = (\varepsilon_1, \dots, \varepsilon_q)$. 
Conditioning on $I_C\!=\!j$ gives a posterior $\bbP^{\cE \mid I_C = j}$ over the exogenous noise.
Intervening gives a new matrix $B_i$ and new exogenous noise $\tilde{\cE}_{\cA_i}$, where $\cA_i$ are the intervention targets when $I_A\!=\!j$.
Defining $\cE'$ = $\tilde{\cE}_{\cA_i} \cup \cE_{\overline{\cA_i}}$, then $\bZ(i)$ equals $(I - B_i)^\inv \cE'$ in distribution.
Thus, $\bbE(\bZ(i) \mid I_C = j) =  U_i \bv_i + U_i' \bw_j$ for $U_i = (I - B_i)^\inv$ and $U_i'$ a masked version of $U_i$, with $\bv_i$ and $\bw_j$ coming from $\bbE[\cE']$.
\end{proof}
\rref{thm:linear} shows that $\bbE[\bY]$ follows a factor model with a fixed action effect.
The fixed effect can be folded into the factor model as follows:
\begin{equation}\label{eq:factorsi}
U_i \bv_i + U'_i \bw_j = \underbrace{\left(\begin{array}{@{}c|c@{}}
  \begin{matrix}
  & U_i \bv_i &\\
  \end{matrix}
  &
  \begin{matrix}
  & U'_i&\\
  \end{matrix}
\end{array}\right)}_{\bbR^{|\cX| \times (1 + |Z|)}}
\onevec{\bw_j}
\end{equation}
Thus, Synthetic Interventions will be consistent for completing the tensor. 
However, we propose an alternative to deal with the fixed effect.

\subsection{Theoretically-motivated extensions of Synthetic Interventions}
The factor model in \rref{eq:factorsi} has more structure than a general rank $|\bZ| + 1$ factor model. We have a fixed effect along rows plus a factor model of rank $|\bZ|$. The best approximation of a one-way fixed effect is given by the mean-over-contexts estimator (see \rref{appendix:oneFE}). Therefore, following the same idea from the NNM-FE \cite{Athey_2021}, we propose the following modification to SI: (1)  subtract the fixed effect, i.e., let $\bD = \bY - \hat{\bY}^{\text{mean-over-contexts}}$, (2) run SI on this matrix and obtain $\hat{\bD}^{\text{SI}}$, (3) return the estimate $\hat{\bY}^{\text{mean-over-contexts}} + \hat{\bD}^{\text{SI}}$.
Intuitively, by removing the exact fixed effect, our factor model would have rank $|\bZ |$ instead of $|\bZ|+1$, so it would be easier to satisfy the linear span inclusion assumption.

\paragraph{Tensor case.} The non-symmetry in the factor model from \rref{eq:factorsi} is even more relevant in the tensor case, i.e. $|\cX|>1$ (we have a matrix that depends on $i$ and a vector that depends on $j$). Therefore, it is more reasonable to regress along contexts than along actions. To show how SI is non-symmetric for tensors, in \rref{appendix:SItensor}, we provide a detailed example for a two-node case where we observe both $Z_1$ and $Z_2$. Following our setup, we provide a $4\times 4 \times 2$ tensor that can only be completed if we use SI within columns (contexts). In particular, for using SI within rows we would need more data to satisfy the linear span inclusion assumption.

\subsection{Exploiting Additional Structure in the Model Class}
One drawback about the factor model from \rref{eq:factorsi} is that it has rank $|\bZ| + 1$, which might be considerably large.
We now consider two additional assumptions which impose additional structure on how $I_A$ and $I_C$ interact with the SCM over $\bZ$.
First, we consider models where the all contexts are defined by some common set of system variables.

\begin{assumption} \label{assumption:context_gaussian}
    There is some $\cC \subset [q]$ such that $\cC_j = \cC$ for all $j \in [n]$.
\end{assumption}
In our setting, where we observe only a very coarse-grained context index, such as cell type, it is reasonable to suppose that (at least approximately) the index can be determined from only a small set of \textit{context markers} $\mathcal{C}$. 
For example, there may exist some small set of marker genes that define cell states.
\rref{assumption:context_gaussian} is important to consider, since it can have a substantial effect on identifiability and estimation by reducing the rank of our factor model to $|\mathcal{C}|$. 

\begin{prop} \label{prop:lineargaussian}
    Let Assumptions \ref{assumption:linear} and \ref{assumption:context_gaussian} hold, and assume that the underlying SCM is Gaussian.
    Then the latent factor model from \rref{thm:linear} is
    \[
    U_i'\bw_j = U_i'W \bc_j
    \]
    where W is a constant matrix and $\bc_j$ is the value that context markers take to define context $j$. Therefore, by defining $U''_i = U'_i W \in \bbR^{|\cX| \times |\cC|}$ we obtain a latent factor model of rank $|\cC|$.
\end{prop}
\begin{assumption}\label{assumption:actions}
    There is some $\cA \subset [q]$ such that $\cA_i = \cA$ for all $i \in [m]$.
\end{assumption}
\rref{assumption:actions} indicates that all actions change the SCM in a similar way. Although this may not be realistic in most settings, we might have this setup for a particular sub-matrix of our whole matrix $\bY$.
For example, \rref{assumption:actions} may hold for drugs which use the same mechanism of actions, or if the different actions correspond to the same drug over different dosage concentrations.
In this case, the latent factor model reduces to a two-way fixed effects model.
%





\begin{cor} \label{cor:fe}
Under Assumptions \ref{assumption:linear} and \ref{assumption:actions} we have
\[
\bbE\left(\bZ_\cX(i) \mid I_C = j \right) = U \bv_i + U' \bw_j.
\]
so Fixed Effect (FE) is a consistent estimator for completing $\bY$.
\end{cor}
In our previous results, we have considered only the case of a linear SCM.
This can be extended in the simplified case where (1) all interventions occur on the same set of nodes and (2) all contexts use the same context-defining nodes. 
%
While linearity may appear to be a strong assumption, it is closely related to the low-rank assumption, which is essential in matrix completion. In \rref{appendix:polynomial}, we study some examples of non-linear SCMs. In particular, when considering polynomials, we observe how the rank of the LFM increases as the degree of the polynomials increases. These results provide evidence that some form of linearity or low-degree polynomial assumption may be required to justify a connection between factor models and causal models.

Finally, in real applications, we need to test the validity of our assumptions, i.e. we need to test whether our tensor comes from a certain graphical model.
So far, we have only exploited the structure that the SCM induces on \textit{expectations} of $\bY$. 
However, we can also use the structure to obtain fine-grained implications regarding the variance of our observations.
This structure on the variance is helpful for inference and for hypothesis testing.
%
%
We demonstrate this idea in \rref{appendix:hypothesis_tests}, where we propose a hypothesis test for the fixed effects model implied by \rref{cor:fe}.
In particular, we test \rref{cor:fe} implies homoscedastic errors within a row.
Thus, we develop a hypothesis test for homoscedasticity, and then a particular test for fixed effects which reduces to a Welch's test.
Similarly, for the latent factor model (LFM) case from \rref{prop:lineargaussian}, we can test homoscedasticity within columns, then use the test for LFMs proposed by \citet{agarwalSI}.
%
%
In \rref{appendix:hypothesis_tests}, we demonstrate the performance of these tests on synthetic data.

\section{Empirical results} \label{sec:experiments}

We work with the PRISM Repurposing dataset \citep{depmap}. 
The entries of $\bY$ measure a viability score, which indicates how lethal a drug is for a given cell line (i.e., negative viability means that a large proportion of cells die).
Since we consider the index-only setting, we use only the observed entries of this matrix to make our predictions.
To relate the PRISM data into the model class introduced in \rref{sec:theory}, we think of the initial cell states in terms of \rref{fig:viability-and-cell-diagram} (left).
In particular, we assume that there is some latent space (defined by the context variables) where we can perfectly distinguish different cell lines, e.g. the expression levels of some marker genes. 

\begin{figure}[t]
    \includegraphics[width=0.45\textwidth]{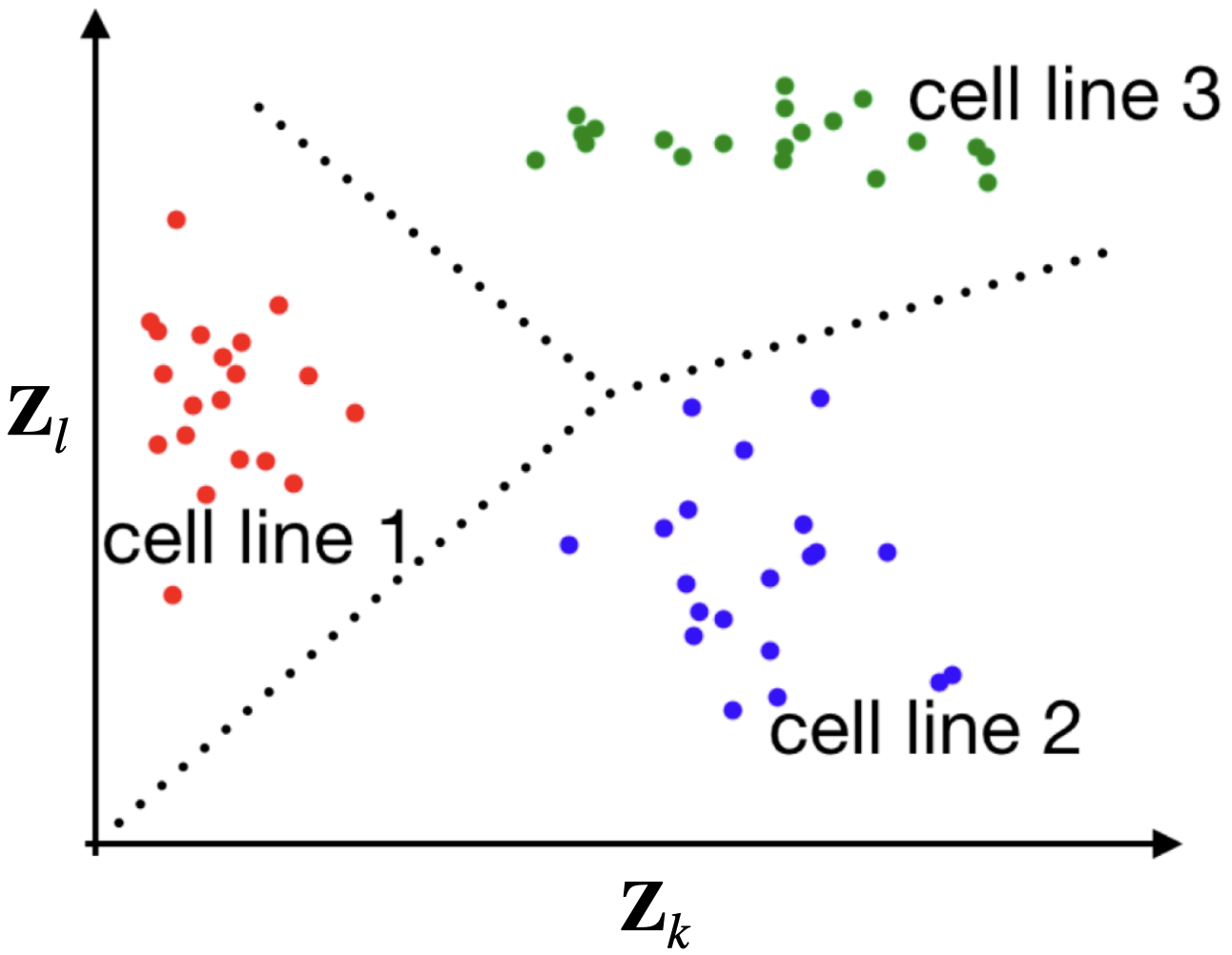}
    \quad
    \includegraphics[width=0.45\textwidth]{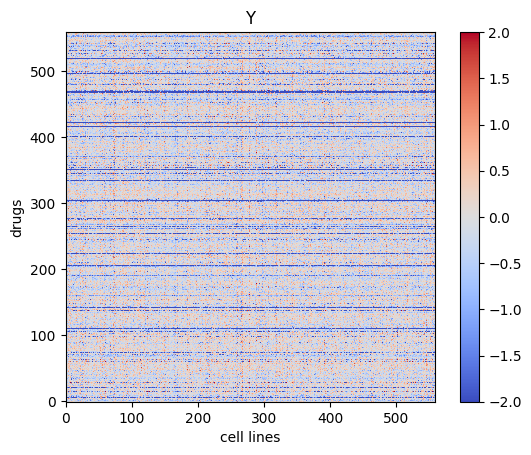}
    \caption{
    \textbf{(Left) Translating the PRISM data to our model class}. Letting $\cC = \{ k, l \}$ indicates that cell state is defined solely in terms of $Z_k$ and $Z_l$.
    \textbf{(Right) Matrix of viability scores}. Each entry represents the viability score for a drug-cell line pair. Negative viability indicates cell death. Viability scores are normalized for visualization purposes.
    }
    \label{fig:viability-and-cell-diagram}
\end{figure}

Arranging drugs in rows and cell lines in columns, as in \rref{fig:matrixY} from \rref{sec:intro}, we obtain a $4,686 \times  568$ matrix. 
To check the symmetry of an algorithm, we want the same number of rows as columns, so that completing along one direction (e.g. columns) is not better simply because there is more data to make predictions. 
Thus, we run the experiments on a square submatrix to avoid of this bias. 
For the same reason, we use symmetric missing data patterns, as described in Section \ref{missingpattern}.
We compare the performance of the algorithms in terms of the $R^2$ score.
The baseline used in the denominator of the $R^2$ is the average over the missing data. 
As an example, for the matrix shown in \rref{fig:viability-and-cell-diagram} (right) and a missing pattern from \rref{fig:missingsquares}, the baseline MSE is approximately $0.87$. 
Getting high $R^2$ is a difficult task because the matrix is relatively noisy (see \rref{appendix:SVDanalysis}).

\begin{figure}[b]
\floatbox[{\capbeside\thisfloatsetup{capbesideposition={right,center},capbesidewidth=0cm}}]{figure}[0.5\textwidth]
{\caption{\textbf{Missing data patterns used in our experiments}. Observed entries are denoted with black. For each missing entry we have the same number of observations along rows and columns.}\label{fig:missingsquares}}
{\includegraphics[width=0.45\textwidth]{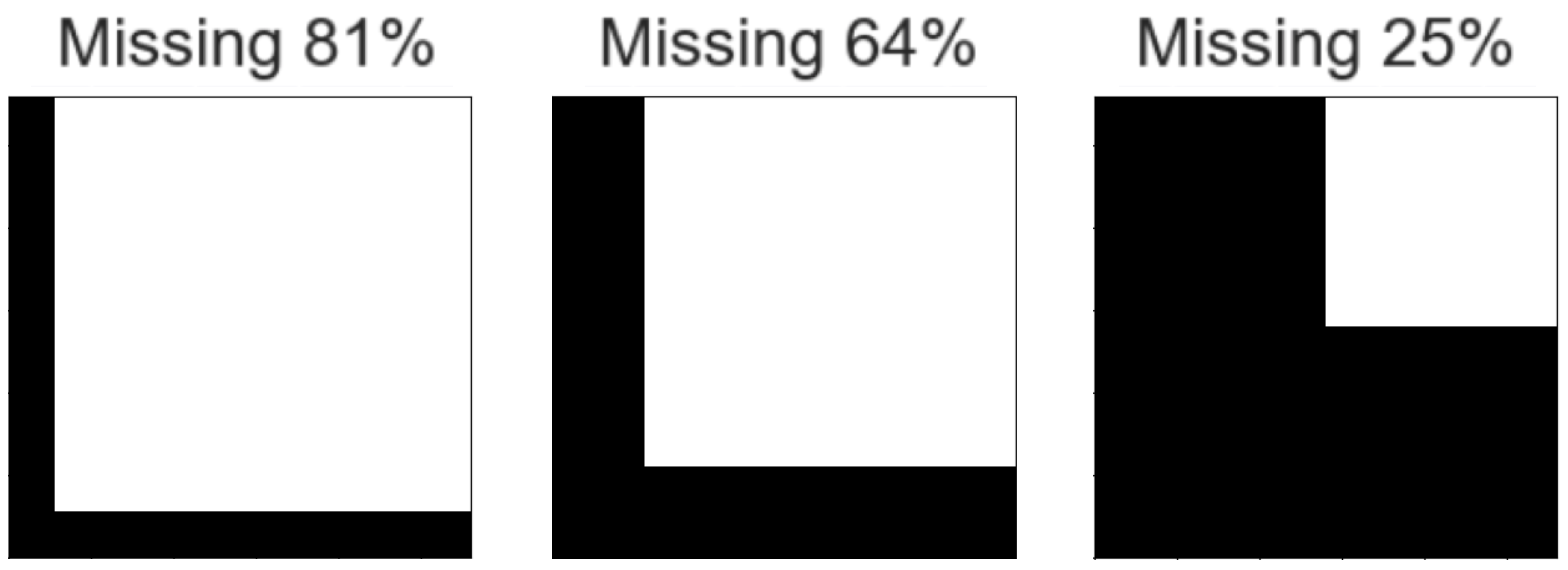}}
\end{figure}



\subsection{About the missing data pattern} \label{missingpattern}

In the PRISM dataset, we have observed all the matrix. As a consequence, we can decide the missing data pattern to test the performance of the matrix completion approaches. But which missing data pattern is more appropriate for biological data?


As we have seen in \rref{sec:relatedwork}, Nuclear Norm Minimization algorithms are near-optimal when we have uniformly at random missing entries. 
However, in biology applications this is usually not the case. Instead, it is more common to have a small portion of cell lines that has been tested against many of the drugs, but for the vast majority of cell lines we have only run a few experiments. For example, in \rref{appendix:missingcmap} we can see the missing data pattern for the CMAP dataset.

For this reason, we test the algorithms on missing patterns as the ones from \rref{fig:missingsquares}. Here, for each missing entry $(i,j)$, the number of non-missing entries (black squares) in the $i$-th row is equal to the number of non-missing entries in the $j$-th column. We run the experiments using different data patterns to see the effect of having less/more available data.

In \rref{appendix:curvedmissing} we consider an alternative missing data pattern where we do not have a constant number of observations for each missing entry. We can observe similar results also in this setting.

\subsection{Performance of matrix completion algorithms} \label{sec:performance}

\begin{figure}[t]
    \centering
    \includegraphics[width=\linewidth]{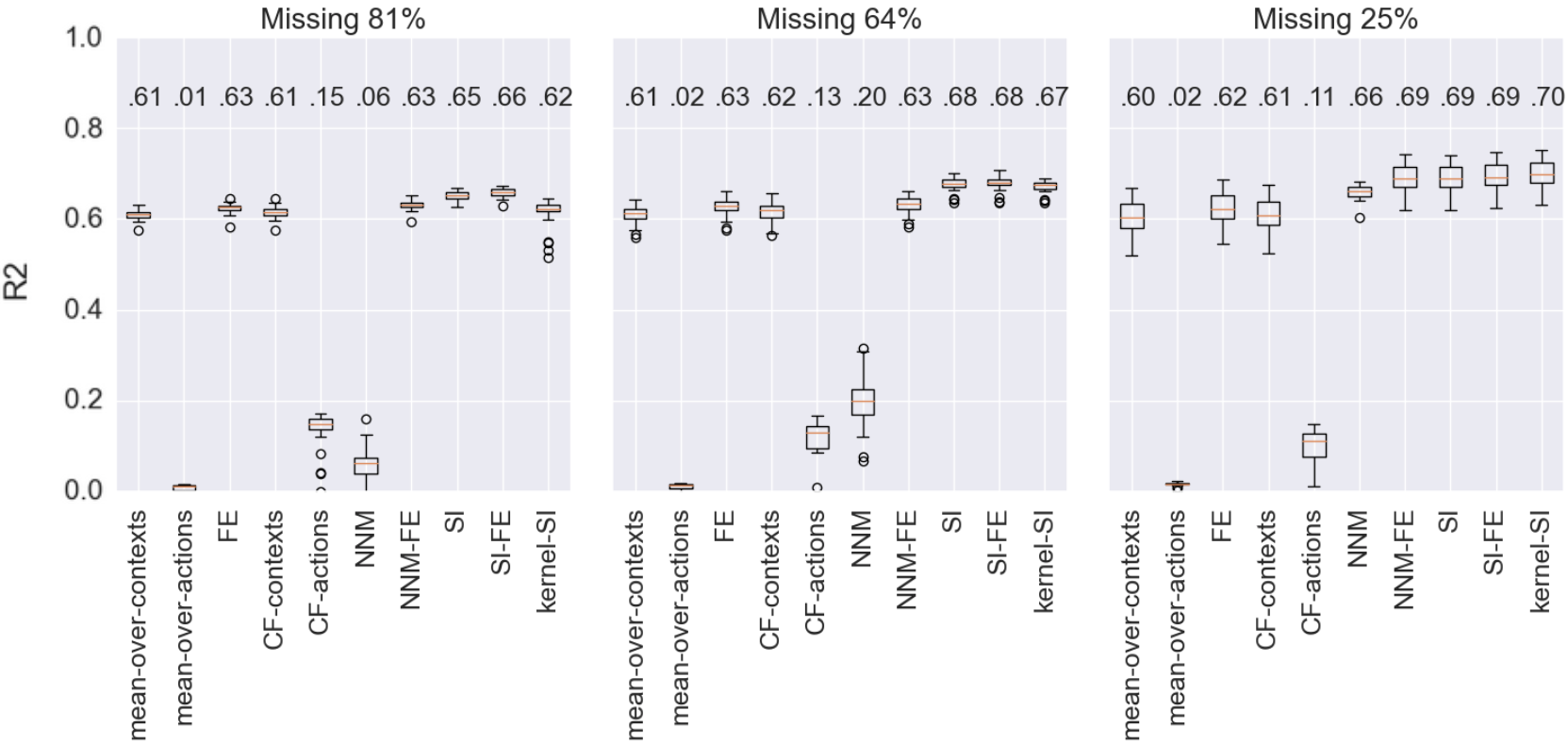}
    \caption{\textbf{Variants of Synthetic Interventions have the best performance on matrix completion for the PRISM dataset.} Using the missing data patterns are ones shown in \rref{fig:missingsquares}, we test the performance of each algorithm on 20 different shuffles of rows and columns. The results are shown using boxplots. The numbers at the top denote the median (red lines).}
    \label{fig:boxplots}
\end{figure}

In this section we show the performance of the matrix completion algorithms presented before.
In particular, we are interested in analyzing how this performance depends on the amount of observed data. An extensive analysis of this dependence can be bound in \rref{appendix:increasingobs}.

In \rref{fig:boxplots} we see that mean-over-contexts is a strong baseline and that Synthetic Interventions (and its variants) outperform all the other approaches. NNM performs competitively when we observe a sufficiently large amount of entries (right-hand side plot) but it performs quite poorly otherwise.

These results reinforce our theoretical findings. In 
\rref{sec:theory} we see how, assuming an underlying linear SCM, the counterfactual quantity of our interest leads to a particular factor model, implying consistency of the Synthetic Interventions estimator.
For further analysis, in \rref{appendix:boxplotsall} we report results for other variations of these approaches. 
In \rref{appendix:killerlung} we run similar experiments on a particular submatrix.
\section{Discussion} \label{sec:conclusion}

In this paper, we generalize the SCM-based model class introduced by \citet{squires2022causal} to allow for situations where the measured outcome is defined as a counterfactual.
We showed that latent factor models (LFMs) naturally arise out of this model class, though with an additional term for fixed action effects.
%
%
%
This demonstrates the fundamentally distinctive nature of \textit{causal} matrix completion, where there is a difference between conditioning (completing within columns) and intervening (completing within rows).
As a practical consequence, this led us to propose a simple extension of the Synthetic Interventions (SI) estimator which includes a fixed effect term.
%
%
These model-inspired causal matrix completion approaches work considerably well, especially in the low-data regime.

\paragraph{Limitations and Future Work.}
In this paper, we largely focused on linear models, and only sought to develop \textit{consistent} estimators for causal predictions.
An important next step is to analyze the noisy case, focusing on statistical aspects such as sample complexity and efficient inference.

A key limitation of our work is that we only considered \textit{index-only} causal prediction problems.
In many applications, additional data is available for actions and/or contexts, e.g., the molecular structure of a drug would be highly relevant to predicting its effect.
We expect that the current work will be a useful conceptual foundation for developing model-based approaches to these causal prediction problems, especially when developing further connections between SCMs and latent factor models.
It would be especially interesting to explore connections with recent works for intervention extrapolation, such as \cite{agarwal2023synthetic} and \cite{saengkyongam2023identifying}.

Finally, while our work was mainly motivated by a biological problem, the causal prediction problem is very general, and it would be interesting to apply our framework in other settings.


\acks{
This work was supported in part by funding from the Eric and Wendy Schmidt Center at the Broad Institute of MIT and Harvard.

{\'A}.~Ribot acknowledges support by the mobility grants program of Centre de Formaci{\'o} Interdisciplin{\`a}ria Superior (CFIS) - Universitat Polit{\`e}cnica de Catalunya (UPC), the MIT Simons MMLS Foundation research grant (6941629), and a MOBINT-MIF grant.

C.~Squires and C.~Uhler acknowledge support by the NSF TRIPODS program (DMS-2022448), NCCIH/NIH (1DP2AT012345), ONR (N00014-22-1-2116), the United States Department of Energy (DOE), Office of Advanced Scientific Computing Research (ASCR), via the M2dt MMICC center (DE-SC0023187), AstraZeneca, the Eric and Wendy Schmidt Center at the Broad Institute, and a Simons Investigator Award. 
}

\bibliography{refs}

\newpage

\clearpage
{
\hypersetup{linkcolor=blue}
\renewcommand{\contentsname}{Contents of Appendix}
\tableofcontents
}
\addtocontents{toc}
{\protect\setcounter{tocdepth}{2}}
\newpage

\section*{Supplemental material}

All code for running the experiments presented in this paper can be found at 

\begin{center}
\href{https://github.com/alvaro-ribot/causal-matrix-completion-PRISM}{https://github.com/alvaro-ribot/causal-matrix-completion-PRISM}
\end{center}
%

\appendix

\section{Summary of matrix completion approaches} \label{appendix:formulas}

In \rref{sec:relatedwork}, we presented the different matrix completion approaches considered in this paper. For completeness, in this section, we provide formulae for such approaches.

We follow the notation used in \cite{squires2022causal}. Let $\bY = (Y_{ij}) \in \bbR^{m \times n}$ be our matrix of interest. Let $\Omega \subset [m] \times [n]$ be the set of observed entries and $\cM  = [m] \times [n] \setminus \Omega$ be the set of missing entries. 
For a given row $i$, let $\cC(i) = \{ j : (i,j) \in \Omega \}$ be the set of column indices $j$ for which we have observed $Y_{ij}$. Similarly, for a column $j$ define $\cR(j) = \{ i : (i,j) \in \Omega\}$. This notation is extended to sets of rows ($\cR$) and sets of columns ($\cC$): $\cC(\cR) = \cap_{i \in \cR} \cC(i)$ and $\cR(\cC) = \cap_{j \in \cC} \cR(j)$. The following example clarifies our notation.
\[
\boxed{
\begin{pmatrix}
        Y_{11} & Y_{12} & ? & ? & Y_{15} \\
        ? & ? & Y_{23} & Y_{24} & ? \\
        Y_{31} & ? & Y_{33} & ? & Y_{35} \\
    \end{pmatrix}
    \quad \quad
\begin{matrix}
        \cC(1) = \{ 1, 2, 5\} \\
          \\
        \cC(3) = \{ 1, 3, 5\} \\
\end{matrix}
\implies 
\begin{matrix}
         \\
        \cC(\{ 1,3 \}) = \{ 1, 5\}  \\
         \\
\end{matrix}
}
\]
\begin{itemize}
    \item Mean-Over-Contexts: $\hat{Y}_{ij}^{\text{mean-over-contexts}} = \frac{1}{|\cC(i)|} \sum_{j' \in \cC(i)} {Y}_{ij'}$.
    \item Mean-Over-Actions: $\hat{Y}_{ij}^{\text{mean-over-actions}} = \frac{1}{|\cR(j)|} \sum_{i' \in \cR(j)} {Y}_{i'j}$.
    \item Fixed Effects (FE): $\hat{Y}_{ij}^{\text{FE}} = \hat{Y}_{ij}^{\text{mean-over-contexts}} + \hat{Y}_{ij}^{\text{mean-over-actions}} - \frac{1}{|\cR(j)||\cC(i)|} \sum_{\substack{i' \in \cR(j) \\ j' \in \cC(i)}}Y_{i'j'}$.
    \item Collaborative Filtering (CF): $\hat{Y}_{ij}^{\text{CF-contexts}} = \frac{1}{\sum_{j' \in \cR(i)} |\text{cos-sim}(j', j)|} \sum_{j' \in \cR(i)} \text{cos-sim}(j', j) Y_{ij'}$.
    \item Synthetic interventions (SI): Let $\cR_{\text{train}} = \cR(\cC(i) \cup \{ j \})$, $\bY_{\text{train}} = [Y_{i'j'} : i' \in \cR_{\text{train}}, j' \in \cC(i)]$, $\by_{\text{output}} = [Y_{i'j} : i' \in \cR_{\text{train}}]$, $\by_{\text{test}} = [Y_{ij'} : j' \in \cC(i)]$. Then solve linear regression:
    \[
    \hat{Y}_{ij}^{\text{SI}} = \by_{\text{test}}^\top \hat{\beta}, \quad \text{where }
    \hat{\beta} = \bY_{\text{train}}^\dagger \by_{\text{output}} \text { and $\dagger$ denotes the pseudoinverse.}
    \]
    \item SI-FE:  Let $\bD = \bY - \hat{\bY}^{\text{FE}}$, then 
    $\hat{Y}_{ij}^{\text{SI-FE}} = \hat{Y}_{ij}^{\text{FE}} + \hat{D}_{ij}^{\text{SI}} $.
    \item Nuclear Norm Minimization (NNM): $\hat{\bY}^{\text{NNM}} = \argmin_\bL \frac{1}{2} \|\bP_\Omega(\bY - \bL) \|^2_\text{F} + \lambda \| \bL\|_{*}$, where $(\bP_\Omega(\bM))_{ij} = M_{ij}$ if $(i,j) \in \Omega$ and $0$ otherwise.
    \item NNM-FE: $\hat{\bY}^{\text{NNM-FE}} = \hat{\bL} + \hat{\Gamma} {\bf 1}_n^{\top} + {\bf 1}_m \hat{\Delta}^{\top}$, where ${\bf 1}_n = (1, \dots, 1) \in \bbR^n$ and
    \[
    (\hat{\bL}, \hat{\Gamma}, \hat{\Delta}) = \argmin_{\bL, \Gamma, \Delta} \frac{1}{|\Omega|} \|\bP_\Omega(\bY - \bL - \Gamma {\bf 1}_n^{\top} - {\bf 1}_m \Delta^{\top} ) \|^2_\text{F} + \lambda \| \bL\|_{*}.
    \]
\end{itemize}
\section{Proofs} \label{appendix:proofs}

\textbf{Notation.} Recall that $[q] := \{1, \dots, q\}$. Given a tuple of indices $\cK = (k_1, \dots, k_p)$ where $k_r \in [q]$ for all $r = 1, \ldots, p$, we define the 
matrix $\bbI_\cK$ as follows:
\[
\mathbb{I}_{\cK} \in \mathbb{R}^{q \times q} \text{ s.t. } (\mathbb{I}_{\cK})_{rs} = \begin{cases}
    1 \text{ if } r = s \text{ and } r \in \cK\\
    0 \text{ otherwise}
    \end{cases}
\]
Furthermore, given a vector $\bZ \in \bbR^q$, we define $\bZ_\cK \in \bbR^{|\cK|}$ such that $(\bZ_\cK)_l = \bZ_{k_l}$. Finally, given a matrix $M \in \bbR^{q\times q}$, we denote by $[M]_\cK \in \bbR^{|\cK| \times q}$ the sub-matrix of $M$ whose rows indices are in $\cK$.
To simplify notation, for any $\bv \in \bbR^d$, we define $\onevec{\bv} = (1, v_1, \ldots, v_d)^\top \in \bbR^{d+1}$ as the vector given by prepending a 1 to $\bv$.

\newtheorem*{T1}{Theorem~\ref{thm:linear}}

\begin{T1}
Under \rref{assumption:linear}, we have
\[
\bbE(\bY_{ij}) = \bbE\left(\bZ_\cX(i) \mid I_C = j \right) = U_i \bv_i + U'_i \bw_j.
\]
for some $U_i, U'_i \in \bbR^{|\cX| \times |Z|}, \bv_i \in \bbR^{|Z|}$ depending on the action index $i$, and some $\bw_j \in \bbR^{|Z|}$ depending on the context index $j$.
\end{T1}

\begin{proof} \label{proof:linear}
Let ${\bf Z} = (Z_1, \dots, Z_q)$ follow a linear SEM, i.e.
\[
    {\bf Z} = B {\bf Z} + \mathcal{E}
\]
where $B = (B_{ij}) \in \mathbb{R}^{q \times q}$ such that $B_{ij} = 0$ whenever $i \leq j$ and $\mathcal{E} = (\varepsilon_1, \dots, \varepsilon_q)$.
Let $I$ be the identity $q \times q$ matrix.
Since $B$ is lower triangular, $(I - B)$ is invertible and we have that 
\[
{\bf Z} = (I - B)^{-1}\mathcal{E}
\]
When we condition on the value of the context index, i.e. on $I_C = j$, we get a new distribution for the noise, $\bbP^{\cE \mid I_C=j}$.
In particular, 
\[
\bw_j := \bbE(\mathcal{E} \mid I_C = j) \in \bbR^q
\]
is a function of $j$.

Now let $\cA_i \subseteq [q]$ be the target set of indices for the intervention with index $i$. For each $k \in \cA_i$, the $k$-th row of $B$ is changed (modifying the dependency from its parents). 
Let $B_i$ be the weight matrix after the intervention\footnote{In the particular case of $\DO$-interventions, we would remove these dependencies, so all the entries in the $k$-th row of $B_i$ would be 0.}.

Moreover, for each $k \in \cA_i$, we have a new noise variable $\tilde{\varepsilon}_k$, independent of $\cE$.
So, the expected value of these new variables depend on the intervention index $i$ but not on the context index $j$. 
Let $\tilde{\cE}$ be a $q$-dimensional random vector such that its $k$-th entry is $\tilde{\varepsilon}_k$ if $k \in \cA_i$ and 0 otherwise. Define
\[
\bv_i := \bbE(\tilde{\cE}) \in \bbR^q,
\quad\quad
\textnormal{and~note~that~}
\bbE(\tilde{\cE}) = \bbE(\tilde{\cE} \mid I_C = j).
\]
%
Therefore, the variables $\bZ(i)$ after the intervention satisfy the following SEM
\[
\bZ(i) = B_i \bZ(i) + \tilde{\cE} + \bbI_{\overline{\cA_i}} \cE \implies \bZ(i) = (I - B_i)^{-1} \left(\tilde{\cE} + \bbI_{\overline{\cA_i}} \cE \right)
\]
where $\overline{\cA_i} = [q] \setminus \cA_i$. Hence,
\[
\bbE(\bZ(i) \mid I_C= j) =  (I - B_i)^{-1} \bbE(\tilde{\cE}) + (I - B_i)^{-1} \bbI_{\overline{\cA_i}} \bbE(\cE \mid I_C= j)
\]
Recall that we are interested only in $\bZ_\cX$ (the variables we observe).
Let
\[
U_i = \left[(I - B_i)^{-1} \right]_\cX \in \bbR^{|\cX| \times q} \quad \quad U'_i = U_i \bbI_{\overline{\cA_i}} \in \bbR^{|\cX| \times q}
\]
Therefore, we have
\[
\bbE(\bZ_\cX(i) \mid I_C = j) = U_i \bv_i + U_i' \bw_j
\]
\end{proof}

\newtheorem*{P1}{Proposition~\ref{prop:lineargaussian}}

\begin{P1}
Let Assumptions \ref{assumption:linear} and \ref{assumption:context_gaussian} hold, and assume that the underlying SCM is Gaussian.
Then the the latent factor model from \rref{thm:linear} is
\[
U_i'\bw_j = U_i'W \bc_j
\]
where W is a constant matrix and $\bc_j$ is the value that context markers take to define context $j$. Therefore, by defining $U''_i = U'_i W \in \bbR^{|\cX| \times |\cC|}$ we obtain a latent factor model of rank $|\cC|$.
\end{P1}

\begin{proof}
We have $\mathcal{E} = (\varepsilon_1, \dots, \varepsilon_q) \sim \mathcal{N}(\mu_\varepsilon ={\bf 0}, \Sigma_\varepsilon = \text{diag}(\sigma_i^2))$, so 
\[
\Sigma := \text{Cov} ({\bf Z}) = (I - B)^{-1} \Sigma_\varepsilon (I - B^{\top})^{-1}
\]
Let $\bZ_\cC$ be the context markers variables and $\overline{\cC} = [q] \setminus \cC$. Using the Schur Complement we obtain
\[
\mathbb{E}({\bf Z}_{\overline{\cC}} \mid \bZ_\cC = \bc_j) = \Sigma_{\overline{\cC}\cC} \Sigma_{\cC \cC}^{-1} \bc_j
\]
Of course, we also have $\mathbb{E}({\bf Z}_{\cC} \mid {\bf Z}_{\cC} = {\bf c}_j) = {\bf c}_j $, so $\mathbb{E}({\bf Z} \mid {\bf Z}_\cC = {\bf c}_j) = M{\bf c}_j$ for some matrix $M$. 
One way to be more specific about this $M$ is the following: Let $P$ be the permutation matrix such that $P {\bf Z} = ({\bf Z}_{\overline{\cC}}, {\bf Z}_{\cC})^{\top}$. Then we have
\begin{equation*}
    \mathbb{E}(P{\bf Z} \mid {\bf Z}_{\cC} = {\bf c}_j) = \begin{pmatrix}
    \Sigma_{\overline{\cC}\cC} \Sigma_{\cC\cC}^{-1} \\
    I_{|\cC|}
    \end{pmatrix}
    {\bf c}_j
    \implies
    M = P^{-1} \begin{pmatrix}
    \Sigma_{\overline{\cC}\cC} \Sigma_{\cC\cC}^{-1} \\
    I_{|\cC|}
    \end{pmatrix}
\end{equation*}
Hence, after conditioning, the expected value of the noise is
\begin{equation*}
     \mathbb{E}(\mathcal{E} \mid {\bf Z}_{\cC} = {\bf c}_j) = (I-B)  \mathbb{E}({\bf Z} \mid {\bf Z}_{\cC} = {\bf c}_j) = (I-B)M{\bf c}_j
\end{equation*}
Therefore, we can define 
\[
W = (I-B)M = (I-B) P^{-1} \begin{pmatrix}
    \Sigma_{\overline{\cC}\cC} \Sigma_{\cC \cC}^{-1} \\
    I_{|\cC|}
    \end{pmatrix} \in \bbR^{q \times |\cC|}
\]
Note that $W$ depends on the nodes we are conditioning on, but not on their values. The rest of the proof is completely analogous to the previous one.
\end{proof}

\begin{remark}
Although it is common to work with $\mu_\varepsilon = {\bf 0}$, there is no need to assume that. In general, again using the Schur complement, we would have (let $\mu = \mu_\varepsilon$)
\begin{align*}
    \mathbb{E}({\bf Z}_{\overline{\cC}} \mid {\bf Z}_{\cC} = \bc_j) &= \Sigma_{\overline{\cC}\cC} \Sigma_{\cC\cC}^{-1}({\bf c}_j - \mu_\cC) + \mu_{\overline{\cC}} = \Sigma_{\overline{\cC}\cC} \Sigma_{\cC\cC}^{-1}{\bf c}_j - \Sigma_{\overline{\cC}\cC} \Sigma_{\cC\cC}^{-1}\mu_\cC + \mu_{\overline{\cC}} \\
    &= \left(\begin{array}{@{}c|c@{}}
          \begin{matrix}
            \mu_{\overline{\cC}} - \Sigma_{\overline{\cC}\cC} \Sigma_{\cC\cC}^{-1}\mu_\cC\\
          \end{matrix}
          & 
          \begin{matrix}
           \Sigma_{\overline{\cC}\cC} \Sigma_{\cC\cC}^{-1}\\
          \end{matrix}
        \end{array}\right)
        \onevec{\bc_j}
\end{align*}
So, if we define $\bc'_j = \onevec{\bc_j} \in \bbR^{1 + |\cC|}$ and $M$ accordingly, all the proof works analogously.
\end{remark}


\section{Best approach for estimating a one-way fixed effect} \label{appendix:oneFE}

In \rref{sec:theory}, we claimed that the best approximation of a one-way fixed effect is given by the
mean-over-contexts estimator. Indeed, consider that following problem
\begin{equation} \label{problem:min_moc}
 \min_{\Gamma \in \bbR^m} f(\Gamma):= \|\bY - \Gamma \mathbf{1}^\top \|_F^2   
\end{equation}
where $\bY \in \bbR^{m \times n}$, $\mathbf{1} = (1, \dots, 1) \in \bbR^n$, and $\| \cdot \|_F$ denotes the Frobenius norm of a matrix. By the linearity of trace we have
\[
f(\Gamma) = \tr(\bY^\top \bY) - 2 \tr(\mathbf{1} \Gamma^\top \bY) + \tr(\mathbf{1} \Gamma^\top \Gamma \mathbf{1})
\]
Using the cyclic property of the trace and taking the gradient of $f$ we get
\[
\nabla f (\Gamma) = -2 \bY \mathbf{1} + 2n \Gamma = \mathbf{0} \iff \Gamma = \frac{1}{n}\bY \mathbf{1}
\]
That is, $\Gamma_i = \frac{1}{n}\sum_j Y_{ij}$, which corresponds to the mean-over-contexts estimator. Since $f$ is strictly convex, this is the solution to \rref{problem:min_moc}.
\section{SI is non-symmetric in the tensor case} \label{appendix:SItensor}

In \rref{sec:theory}, we argued that SI is non-symmetric in the tensor case. 
The following example shows that this is true. 
In particular, completing within columns recovers the desired outcome while completing within rows does not. 
The key idea is that we have to flatten the tensor in one direction or another, and depending on the direction, linear span inclusion does or does not hold. Consider the following 3-order tensor
\[
\resizebox{\hsize}{!}{$
\bY =
\left[
\begin{array}{cccc||cccc}
Y_{111} & Y_{121} & Y_{131} & Y_{141} & Y_{112} & Y_{122} & Y_{132} & Y_{142} \\
Y_{211} & Y_{221} & Y_{231} & Y_{241} & Y_{212} & Y_{222} & Y_{232} & Y_{242} \\
Y_{311} & Y_{321} & Y_{331} & Y_{341} & Y_{312} & Y_{322} & Y_{332} & Y_{342} \\
Y_{411} & Y_{421} & Y_{431} & Y_{441} & Y_{412} & Y_{422} & Y_{432} & Y_{442} \\
\end{array}
\right] 
= \left[
\begin{array}{cccc||cccc}
1 & 1 & 1 & 1 & 0 & 2 & 1 & 1 \\
1 & 0 & 1 & 0 & 1 & 1 & 1 & 1 \\
1 & 1 & 1 & 1 & 1 & 1 & 1 & 1 \\
0 & 0 & 0 & 0 & 1 & 1 & 1 & 1 \\
\end{array}
\right]$}
\]
For the $(i.j)$-th entry, we use the notation $\bY_{ij} = (Y_{ij1}, Y_{ij2})\in \bbR^2$. Suppose we have observed all the matrix except for the $(4,4)$-th entry. 
By regressing along rows, we are predicting the (flattened) first three rows of the 4th column $(1, 0, 1, 1, 1, 1)$ from the (flattened) first three rows of the 1st, 2nd and 3rd columns $(1, 1, 1, 0, 1, 1)$, etc.
Thus, we obtain the following system of equations:
\[
\begin{pmatrix}
1 & 1 & 1\\
1 & 0 & 1\\
1 & 1 & 1\\ 
0 & 2 & 1\\
1 & 1 & 1\\
1 & 1 & 1\\
\end{pmatrix}
\beta = 
\begin{pmatrix}
1\\
0\\
1\\
1\\
1\\
1\\
\end{pmatrix}
\]
This linear system is independent and we obtain $\hat{\beta} = (1,1,-1)$. Thus, our prediction for the $(4,4)$-th entry is
\[
\hat{\bY}_{44} = 
\begin{pmatrix}
0 & 0 & 0\\
1 & 1 & 1\\
\end{pmatrix}
\hat{\beta} = 
\begin{pmatrix}
0\\
1\\
\end{pmatrix}
= \bY_{44}
\]
Thus, regressing along rows works successfully. However, by regressing along columns we would obtain the following:
\[
\begin{pmatrix}
1 & 1 & 1\\
1 & 0 & 1\\
1 & 1 & 1\\ 
0 & 1 & 1\\
2 & 1 & 1\\
1 & 1 & 1\\
\end{pmatrix}
\alpha = 
\begin{pmatrix}
0\\
0\\
0\\
1\\
1\\
1\\
\end{pmatrix}
\]
which is an inconsistent system. The least-squares solution is $\hat{\alpha} = (0,0.6,0)$, so the prediction would be
\[
\hat{\bY}_{44} =
\begin{pmatrix}
1 & 0 & 1\\
1 & 1 & 1\\
\end{pmatrix}
\hat{\alpha} = 
\begin{pmatrix}
0\\
0.6\\
\end{pmatrix}
\neq \bY_{44}
\]
So this shows how SI works only in one direction. Moreover, this example can be related to our causal setup. Let $Z_1 \to Z_2$ follow the SCM
\[
\begin{cases}
    Z_1 = \varepsilon_1 \\
    Z_2 = Z_1 + \varepsilon_2
\end{cases}
\]
where $\varepsilon_1, \varepsilon_2 \sim \cN(0,1)$ are independent noise variables. Suppose that we are observing both variables. Each column of $\bY$ corresponds to conditioning on $\bZ = (c_1, c_2)$. In particular, we have $\bbE \left[ (\varepsilon_1, \varepsilon_2) \mid (Z_1, Z_2) = (c_1, c_2) \right] = (c_1, c_2 - c_1)$. The rows of $\bY$ correspond to an intervention $\DO(Z_1 = a_1, Z_2 = a_2)$. The values of $(c_1, c_2)$ and $(a_1, a_2)$ are defined in \rref{tab:tensorexample},
\begin{table}[h!]
    \centering
    \begin{tabular}{c|cc}
        column $(I_C)$ & $c_1$ & $c_2$ \\
        \hline
        $j = 1$ & 1 & 0\\
        $j = 2$ & 0 & 1\\
        $j = 3$ & 1 & 1\\
        $j = 4$ & 0 & 0\\
    \end{tabular}
    \quad \quad
    \begin{tabular}{c|cc}
        row $(I_A)$ & $a_1$ & $a_2$ \\
        \hline
        $i = 1$ & 1 & -\\
        $i = 2$ & - & 1\\
        $i = 3$ & 1 & 1\\
        $i = 4$ & 0 & 1\\
    \end{tabular}
    \vspace{.2cm}
    \caption{Values used for defining $i$-th row and $j$-th column of $\bY$.}
    \label{tab:tensorexample}
\end{table}
where "-" denotes that we are not intervening on that variable. It is easy to check that $\bbE(\bZ(i) \mid I_C = j) = \bY_{ij}$. Therefore, our matrix comes from a linear causal model. To see why SI only worked on one direction, recall that \rref{eq:factorsi} tells us that we can write
\[
\bY_{ij} = P_i \bq_j
\]
for some $P_i \in \bbR^{2 \times 3}$ and $\bq_j = \bbR^{3}$. Following the Proof of \rref{thm:linear}, we can construct these factors:
\begin{align*}
    &P_1 = \begin{pmatrix}
        1 & 0 & 0\\
        1 & 0 & 1\\
    \end{pmatrix}
    &P_2 = \begin{pmatrix}
        0 & 1 & 0\\
        1 & 0 & 0\\
    \end{pmatrix} \quad \quad
    &P_3 = \begin{pmatrix}
        1 & 0 & 0\\
        1 & 0 & 0\\
    \end{pmatrix}
    &P_4 = \begin{pmatrix}
        0 & 0 & 0\\
        1 & 0 & 0\\
    \end{pmatrix} \\[0.1cm]
    &\bq_1 = \begin{pmatrix}
        1 & 1 & -1\\
    \end{pmatrix}
    &\bq_2 = \begin{pmatrix}
        1 & 0 & 1\\
    \end{pmatrix} \quad \quad
    &\bq_3 = \begin{pmatrix}
        1 & 1 & 0\\
    \end{pmatrix}
    &\bq_4 = \begin{pmatrix}
        1 & 0 & 0\\
    \end{pmatrix}
\end{align*}
Note that $\bq_4 \in \operatorname{span}(\bq_1, \bq_2, \bq_3)$. In fact, $\bq_4 = \hat{\beta}_1 \bq_1 + \hat{\beta}_2 \bq_2 + \hat{\beta}_3 \bq_3$, so the linear span inclusion is satisfied along columns. Nevertheless, $P_4 \notin \operatorname{span}(P_1, P_2, P_3)$, so the linear span inclusion is not satisfied along rows.

\section{Polynomial case} \label{appendix:polynomial} 

In \rref{sec:theory}, we discussed the possibility of extending our results to non-linear SCMs. Here, we show an example where the latent variables follow a linear SCM and there is a polynomial link function mapping the latent variables to the observed ones. For simiplicity in notation we use $X = \bZ_\cX$, $A = \bZ_\cA$, and $C = \bZ_\cC$. Note that $\cA$ and $\cC$ are the same for all rows and columns, respectively.

\begin{example} \label{poly_example}
Consider the following SCM:

\begin{equation*}
\begin{cases}
    A = \varepsilon_A \\
    C = B_{AC}A + \varepsilon_C\\
    X = f(A,C) + \varepsilon_X, \quad f \in \mathbb{R}_d[Y_1,Y_2]
\end{cases}
\end{equation*}
Since there is linear relation between $A$ and $C$, we can still use the Schur Complement. For a gaussian variable $Z\sim \mathcal{N}(\mu, \sigma^2)$, we have $\mathbb{E}(Z^k) = p_k(\mu, \sigma)$ for some polynomial $p_k \in \mathbb{R}_k[T_1, T_2]$. Therefore, $\mathbb{E}(X(a_i) | C = c_j)$ is a polynomial in $a_i$ and $c_j$. In particular, we can write it as

\begin{align*}
    \mathbb{E}(X(a_i) | C = c_j) &= q_0(c_j) + q_1(c_j)a_i + \cdots + q_{d-1}(c_j) a_i^{d-1}  + q_d(c_j) a_i^d\\
    & = \begin{pmatrix}
    1 & a_i & \cdots & a_i^d
    \end{pmatrix}
    \begin{pmatrix}
    q_0(c_j) \\
    q_1(c_j)\\
    \vdots \\
    q_d(c_j)
    \end{pmatrix}
\end{align*}
for some polynomials $q_k$.\footnote{In this case, we could defined this polynomial more explicitly as a function of $\bbE(\varepsilon_c^k \mid C = c_j)$, but this is not necessary for expressing the factor model.} Therefore, for a polynomial of degree $d$ we have a factor model of degree $d+1$, so we could also use Synthetic Interventions in this case.
\end{example}
If we had multiple fixed action nodes $A_1, \dots, A_r$ and a polynomial of degree $d$ from the latent nodes to $X$, we could use the same idea and express $\mathbb{E}(X(a_i) \mid C = c_j)$ as a polynomial on $a_i^1, \dots, a_i^r$. 
The number of monomials in $r$ variables of degree no greater than $d$ is $\binom{r + d}{d}$, so that would be the rank of our factor model. 
However, it is not clear how our decomposition would look like if allowed intervening on different nodes, because the $q_k$'s would depend on these indices. We leave this question open as a future direction.

\subsection{Going beyond polynomials}
\begin{example} \label{exponential}
Consider now the following SCM:
\begin{equation*}
\begin{cases}
    A = \varepsilon_A \\
    C = A + \varepsilon_C\\
    X = \exp{(AC)}
\end{cases}
\end{equation*}
After conditioning and intervening, we get that
\[
X(a) | C = c \sim \text{Lognormal}\left(a \left(a + \frac{c}{2} \right), \frac{a^2}{2}\right)
\]
Hence,
\begin{equation*}
    \mathbb{E}(X(a) | C = c) = \exp\left(a \left(a + \frac{c}{2} \right) + \frac{a^2}{4} \right) = \exp \left(\frac{5}{4}a^2 \right) \exp \left(\frac{ac}{2}\right)
\end{equation*}
But this expression cannot be expressed as a low-rank factor model.
\end{example}
Example \ref{exponential} suggests that it might be difficult to obtain low-rank factor models for SCMs that involve functions "more complicated" than polynomials. Moreover, Example \ref{poly_example} shows that the rank of the factor model increases as the degree of the polynomial increases.

Note that the expression obtained is the exponential of a polynomial in $a, c$. If we applied an entry-wise logarithm to our matrix, we could apply the SI estimator and then take the exponential of the outcome observed. With this idea in mind, we could think about developing SI theory for feature spaces. However, this goes beyond the scope of this work and we leave it for future discussion.


\section{Hypothesis tests} \label{appendix:hypothesis_tests}

In \rref{sec:theory}, we discussed the necessity to test the validity of our assumptions. Here, we show how our causal modeling may also be useful for constructing hypothesis tests. For simplicity, assume that we have observed $n_s$ independent samples for each observed entry, i.e. for each $(i,j) \in \Omega$. Assume that $Y_{ij} \sim \mathcal{N}(\mu_{ij}, \sigma_{ij}^2)$  and let $\overline{Y}_{ij} \sim \mathcal{N}(\mu_{i,j}, \frac{\sigma_{ij}^2}{n_s})$ denote the average of these samples. We can estimate the sample variance as follows $S_{ij}^2 = \frac{1}{n_s-1}\sum_{k=1}^{n_s}(Y_{ij}^{(k)} - \overline{Y}_{ij})^2$.


\subsection{Homoscedasticity within rows/columns}

Before discussing the different decompositions obtained in \rref{sec:theory}, is it worth noting that the SCMs we studied give rise to certain structure in the variance of our matrix. In particular, \rref{prop:lineargaussian} leads to homoscedasticity within the entries of the same row ($\sigma_{i,j} = \sigma_{i,j'}$). Indeed, looking at the Schur Complement, the variance of the noise after conditioning on some variables depends on the variables we condition on, but not on the actual values they take. On the other hand, in \rref{cor:fe} we have the same SCMs after the interventions for different rows, which implies homoscedasticity within columns ($\sigma_{i,j} = \sigma_{i',j}$).

This lead to a classical test for comparing the variances of two independent samples. For example, for testing homoscedasticity within column $j$ we consider

\begin{equation*}
    \begin{cases}
    H_0: \frac{\sigma_{i,j}^2}{\sigma_{i-1,j}^2} = 1 \quad \forall~ i\geq2,j \\
    H_1: \frac{\sigma_{i,j}^2}{\sigma_{i-1,j}^2} \neq 1, \quad \exists~ i,j \\
    \end{cases}
\end{equation*}
So we can use the following estimator
\[
F = \frac{S_{i,j}^2}{S_{i-1,j}^2}
\]
and we have $F \mid H_0 \sim F_{n_s - 1, n_s - 1}$. A test for homoscedasticity within row $i$ would be analogous.

\subsection{Test for fixed effects}

Suppose that we are in the case of Corollary \ref{cor:fe}. How can we test if FE is appropriate for our matrix? For simplicity in notation, suppose that we have observed all the entries of $ \bY \in \bbR^{m\times n}$. For a partially observed entries, we would use these tests for the observed entries. With this notation, all the matrices with a two-way fixed effect are defined by the following equations.
\begin{equation}\label{eq:hypothesis}
    Y_{ij} + Y_{i'j'} - Y_{i'j} - Y_{ij'} = 0 \quad \forall~ i \neq i', j \neq j'
\end{equation}
Many of this equations are linearly dependent. In fact, all of them can be expressed as a linear combination of equations
\begin{equation*} \label{eq:hypo1}
    E_{ij} : Y_{11} + Y_{ij} - Y_{i1} - Y_{1j} = 0, \quad i = 2, \dots, m, j = 2, \dots n\
\end{equation*}
Indeed, by computing $E_{ij} + E_{i'j'} - E_{i'j} - E_{ij'}$ we get \rref{eq:hypothesis}. So we only need to check $(m-1)(n-1)$ equations instead of $m(m-1)n(n-1)$. 
However, one drawback about using \rref{eq:hypo1} for constructing hypothesis tests is that they heavily rely on the (1,1)-th entry. So we can consider a more robust set of equations:
\begin{equation*}
    \Tilde{E}_{ij} : Y_{i-1, j-1} + Y_{i,j} - Y_{i,j-1} - Y_{i-1,j} = 0, \quad i = 2, \dots, m, j = 2, \dots ,n
\end{equation*}
Let $\mu_{i,j} = \bbE(Y_{i,j})$ , the hypothesis test for FE is the following.
\begin{equation*}
    \begin{cases}
    H_0: \mu_{i-1,j-1} + \mu_{i,j} - \mu_{i,j-1} - \mu_{i-1,j} = 0, \quad \forall~ i,j \\
    H_1: \mu_{i-1,j-1} + \mu_{i,j} - \mu_{i,j-1} - \mu_{i-1,j} \neq 0, \quad \exists~ i,j \\
    \end{cases}
\end{equation*}
Therefore, we can consider the following statistic:
\[
\hat{T}_{ij} = \frac{\overline{Y}_{i-1,i-1} + \overline{Y}_{i,j} - \overline{Y}_{i,j-1} - \overline{Y}_{i-1,j}}{\sqrt{\frac{1}{n_s}\left(S_{i-1,j-1}^2 + S_{i,j}^2 + S_{i,j-1}^2 + S_{i-1,j}^2 \right)}}
\]
The key idea is that, in \rref{cor:fe}, we have homoscedasticity within rows. This makes our test similar to a Welch's test, and we can approximate the distribution of $\hat{T}_{i,j} \mid H_0$ as $t_{\hat{\nu}}$ where

\[
\hat{\nu} = (n_s - 1)\frac{\left( (S_{i,j}^2 + S_{i-1,j}^2) + (S_{i,j-1}^2 + S_{i-1,j-1}^2) \right)^2}{(S_{i,j}^2 + S_{i-1,j}^2)^2 + (S_{i,j-1}^2 + S_{i-1,j-1}^2)^2}
\]

\paragraph{Multiple comparisons problem.} We have $N = (n-1)(m-1)$ different null hypothesis, corresponding to each $(i,j)$ pair for $i,j \geq 2$. If we want to achieve a familywise error rate (FWER) of $\alpha$, we can reject a null hypothesis if we get a p-value lower than $\frac{\alpha}{N}$ (Bonferroni correction), or $1 - (1-\alpha)^N$ (Šidák correction).

\paragraph{Empirical Results.} In \rref{fig:ROC} we show the how this estimator is useful for testing a fixed-effect behaviour. In particular, for each plot, we simulate 200 matrices satisfying the assumptions from \rref{cor:fe} and 200 matrices satisfying the assumptions from \rref{prop:lineargaussian}. More specifically, we use the \texttt{causaldag} Python library to generate a DAG with random weights (we use the default parameters). We fix the node we observe ($\bZ_\cX$) and the context variables $\bZ_\cC$. For each simulated matrix, we condition on a different set of values ($\bz_\cC$), sampled from a uniform distribution over $[0, 10]$. The nodes where we intervene on change for each matrix. We perform a $\DO$-intervention on each node with probability $0.2$ and the values we used are also sampled from a uniform distribution over $[0, 10]$. For the matrices where the fixed effect behavior is not satisfied, we do not intervene on the same nodes for all the matrix, so we follow the procedure mentioned before for each row.

We observe how, as the size of the matrix and the number of observed samples increase, the statistic we propose can fully distinguish the two types of matrices.

\begin{figure}[ht]
    \centering
    \includegraphics[width = .7\textwidth]{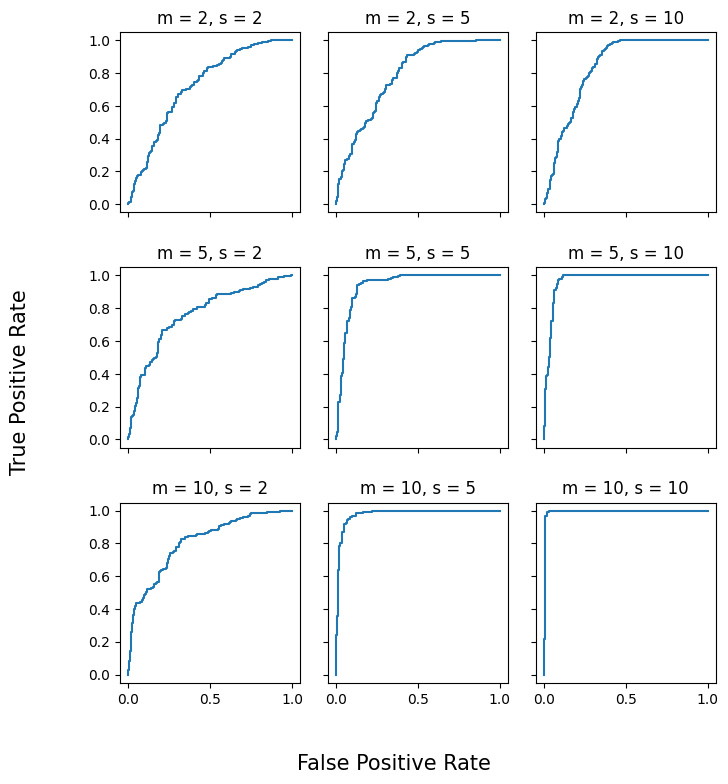}
    \caption{\textbf{ROC curves}. \textit{m} denotes the size of the matrix, i.e. $\bY \in \bbR^{m \times m}$. \textit{s} denotes the number of samples observed per each entry.}
    \label{fig:ROC}
\end{figure}

\subsection{Test for mean-over-contexts}
To be precise, before testing whether or not we our matrix follows a two-way fixed effect model, we should test whether it has only a one-way fixed effect. That is, test if we can use mean over actions or contexts. For example, to test whether we only have a fixed effect along rows, instead of Equation \ref{eq:hypothesis} we would have
\[
Y_{ij} - Y_{ij'} = 0 \quad \forall~ i, j \neq j'
\]
And we can repeat all the reasoning from the previous section to obtain an appropriate estimator. The case of a fixed effect along columns would be completely analogous.

\subsection{Test for Synthetic Interventions}
\cite{agarwalSI} propose a hypothesis test for the SI estimator, so we could use it for our particular case to check the linear span inclusion assumption. 
Moreover, as mentioned before, assuming Gaussianity as in \rref{prop:lineargaussian} we can also test homoscedasticity within rows. Therefore, constructing such tests is an interesting future direction.

\newpage
\section{Further analysis on PRISM Repurposing dataset}

\subsection{SVD analysis} \label{appendix:SVDanalysis}

In \rref{sec:experiments}, we argued that getting high $R^2$ was difficult because the PRISM data is relatively noisy. Here we analyze the Singular Value Decomposition (SVD) of our matrix $\bY$ shown Figure \ref{fig:viability-and-cell-diagram}. It is important to mention that here we are showing the SVD computed on all the matrix. That is, we are also using the entries that were supposed to be missing and that we want to impute as a matrix completion task. Therefore, the results shown in this section are merely an analysis of our data, but we are not going to use them for the imputation task.

In Figure \ref{fig:svalues}, we show the Singular Values of $\bY$. We can see that they are relatively diffuse, suggesting that there is a considerably amount of noise in the data. For example, the explained variance by the top $10$ SVD vectors $\left( \frac{\sum_{i \leq 10} \sigma_i^2}{\sum_{i} \sigma_i^2}\right)$
is just $78\%$.

The Eckart-Young theorem \citep{eckartyoung} states that the best rank $r$ approximation to a matrix is given by truncating its SVD to the top $r$ singular values. We can use this result to establish upper bounds in performance that we should expect from matrix completion approaches. For instance, mean-over-contexts produces a rank 1 approximation of our matrix, so this approximation must be worse (i.e. smaller $R^2$) than the rank 1 truncated SVD.

In Figure \ref{fig:svalues}, we also show the $R^2$ values for the truncated SVD using the missing data pattern like the one from Figure \ref{fig:missingsquares}, with $25\%$ of missing entries. It is important to remark that in this case we are computing the SVD using all the matrix $\bY$ (including missing entries), so this is not a valid completion approach. It should be understood as the optimal performance we can expect.

Generally, it is helpful to see how the first singular values pairs look like. Let $\bS_1$ be rank 1 truncated SVD of $\bY$. In \rref{fig:svd} we observe that $\bS_1$  has almost constant rows, so it seems to be capturing the fixed effect in that direction. This indicates that the fixed effect along rows is much more important than the fixed effect along columns. Therefore, we may expect that the mean-over-contexts is going to be a considerably good baseline. We also see how in $\bY-\bS_1$  we lose a big part of this fixed effect along rows, and our matrix seems to have less variability. This suggests that subtracting the mean-over-contexts may lead us to better predictions (see \rref{sec:theory}). Finally, SVD2 also seems to capture some of this fixed effect behavior, but in a smaller order of magnitude.

\begin{figure}[ht]
    \centering
    \includegraphics[width=.4\linewidth]{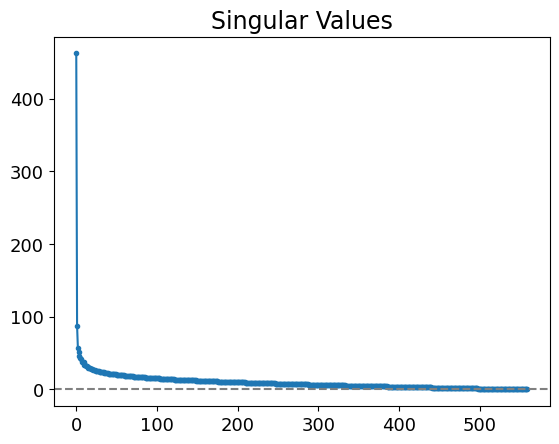}
    \quad
    \includegraphics[width=.4\linewidth]{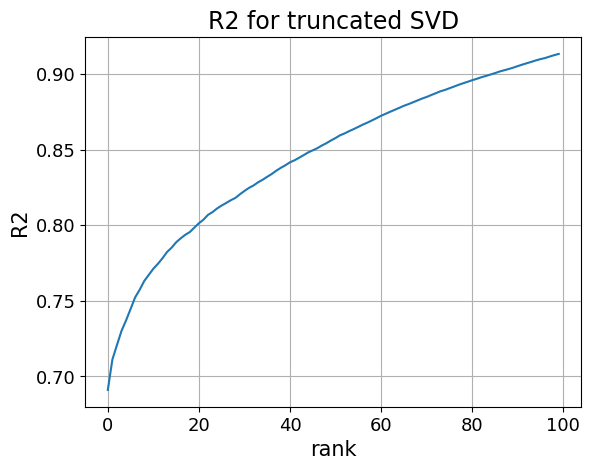}
    
    \caption{\textbf{Difficulty in getting high $R^2$ values}. The singular values of $\bY$ present a relatively heavy tail. This makes it difficult to obtain good metrics in our predictions. The truncated SVD of rank $r$ is the best rank $r$ approximation of a matrix. Therefore, the $R^2$ obtained by the truncated SVD may be considered as an upper bound for our completion algorithms. We should not expect to achieve an $R^2$ of 0.90 as this would involve finding the best approximation of rank $80$.
    }
    \label{fig:svalues}
\end{figure}


\newpage

\begin{figure}[ht]
      \includegraphics[width=0.45\textwidth]{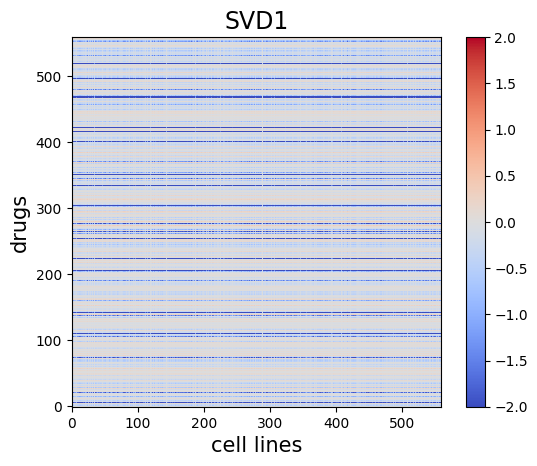}
        \label{fig:svd1}
        \quad 
        \includegraphics[width=0.45\textwidth]{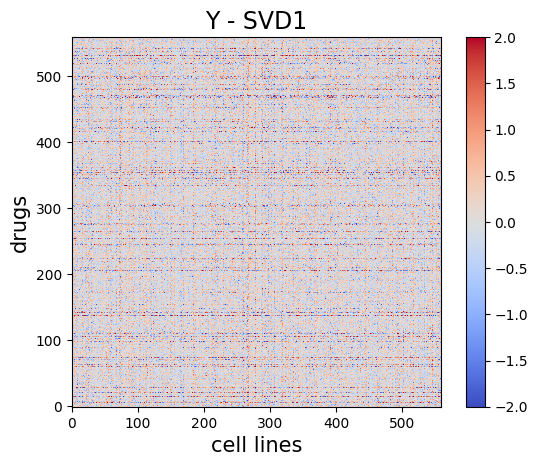}

    \caption{\textbf{One-way fixed effect behavior}. The truncated SVD of rank $1$ (SVD1 = $\bS_1$) has approximately constant rows. This suggests that there is a fixed effect for actions so mean-over-contexts should be a relatively good estimator.}
    \label{fig:svd}
\end{figure}

\subsection{Missing pattern in CMAP dataset} \label{appendix:missingcmap}

As we mentioned in Section \ref{missingpattern}, in biology applications we usually do not have a missing at random pattern. In \rref{fig:missingcmap} we show the entries observed for the CMAP dataset. We use this example as a starting point for the missing data pattern considered in Section \ref{missingpattern}.

\begin{figure}[ht]
    \centering
    \includegraphics[width=0.4\linewidth]{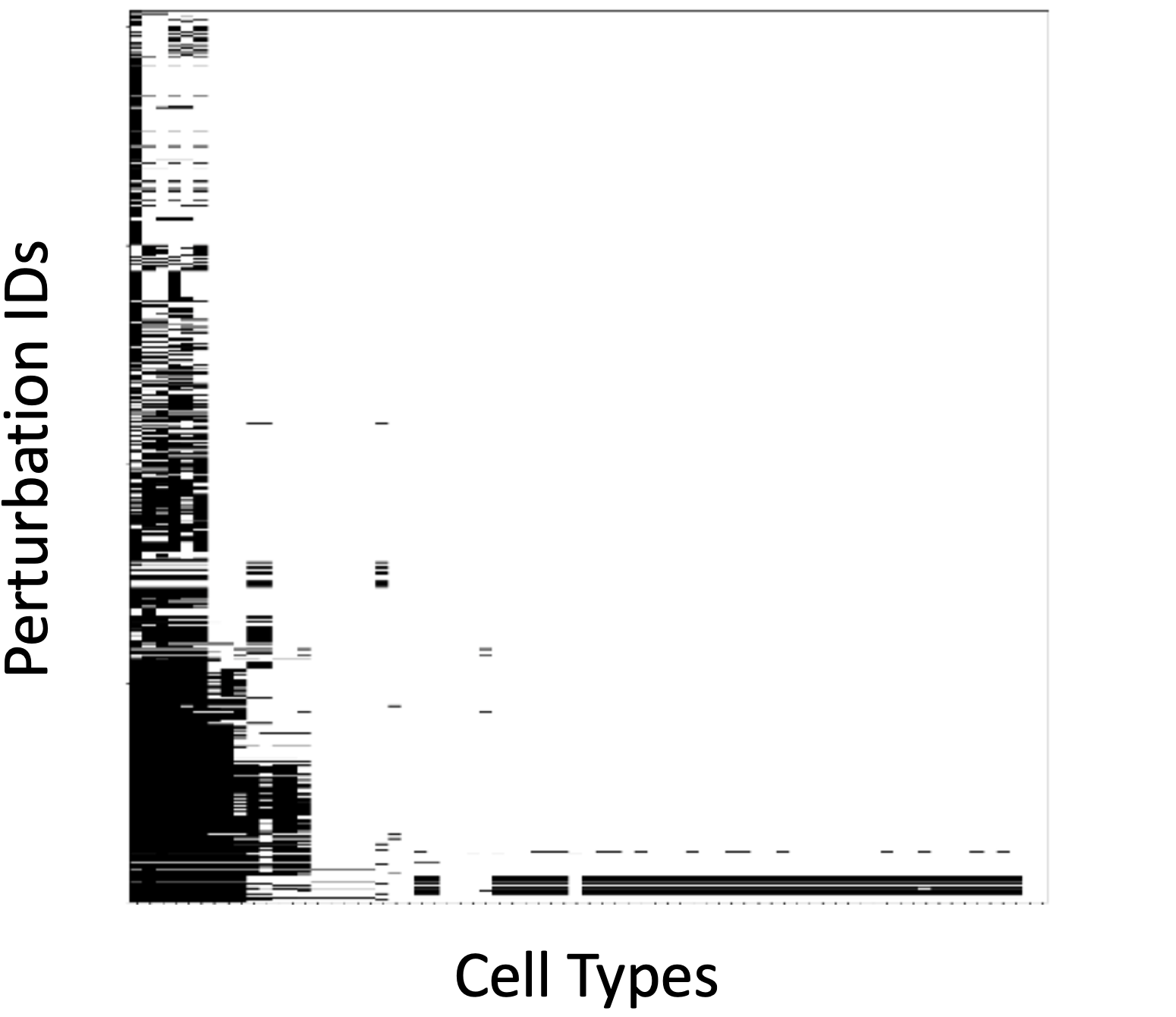}
    \caption{\textbf{Missing pattern NOT at random.} Availability matrix for CMAP dataset. A black rectangle means that we have observed that entry. Columns have been sorted from left to right according to the number of available entries (for each column). In a similar way, rows have been sorted from bottom to top. Figure obtained from \cite{squires2022causal}.}
    \label{fig:missingcmap}
\end{figure}

\subsection{An alternative missing data pattern}\label{appendix:curvedmissing}

The missing data patterns from Section \ref{missingpattern} may seem very artificial, as we have a constant number of observed rows and columns for each missing entry. In \rref{fig:missingcurve} we propose an alternative pattern, still following the motivation from the CMAP dataset (see \rref{appendix:missingcmap}), and we show the performance of different algorithms for this setting. The results are not significantly different to the ones from \rref{fig:boxplots}.

\begin{figure}[h!]
    \centering
    \includegraphics[width=0.38\linewidth]{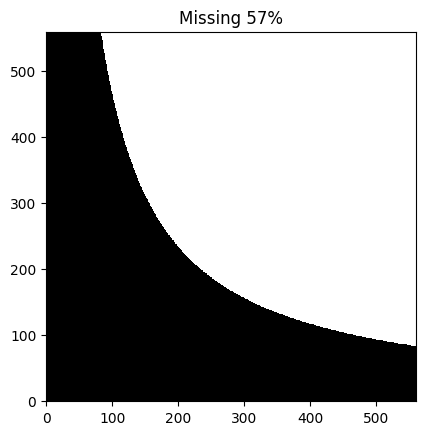}
    \quad
    \includegraphics[width=.4\linewidth]{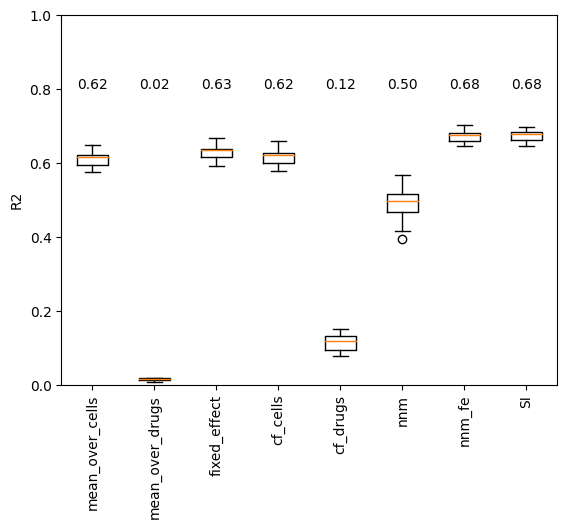}
    \caption{\textbf{Non-square missing data pattern.} Here, we do not have a constant value of rows and columns for each missing entry.}
    \label{fig:missingcurve}
\end{figure}



\subsection{Performance of the algorithms depending on the amount of data observed} \label{appendix:increasingobs}

In \rref{sec:performance}, we showed the performance of different matrix completion approaches for a specific number of observed entries. The experiment we show here consists on varying the number of rows and columns observed per missing entry. This parameter will range from $5$ until $275$.
For each case, the square corresponding to the missing values has constant size $284 \times 284$. The reason for having constant size square (same number of missing entries) is that the MSE is computed on the same number of samples.

We have used the following bootstrap technique. For each case we shuffle the rows and columns of the full matrix and take a sub-matrix of our desired size (e.g. $289\times289$). Then, we take the corresponding missing data pattern and compute the $R^2$ for each algorithm. We repeat the same process for 20 different shuffles. See \rref{fig:increasingobs}. As a takeaway, it is worth noting that Synthetic Interventions benefits from subtracting the mean-over-contexts, especially in the low data regime.

Note that we can think about \rref{fig:boxplots} as taking slices from \rref{fig:increasingobs}. In particular, \rref{fig:boxplots} shows that Collaborative Filtering approach does not improve the baseline algorithms and Kernel Linear Regression is not significantly better than the standard Synthetic Interventions approach. For that reason, we do note include them in \rref{fig:increasingobs}.

\begin{figure}[ht]
\vspace*{-1cm}
    \centering
    \begin{subfigure}{}
        \centering
        \includegraphics[width=.4\textwidth]{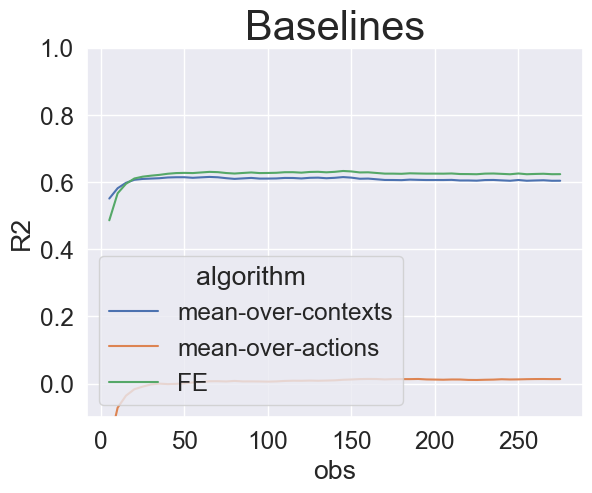}
    \end{subfigure}
    \begin{subfigure}{}
        \centering
        \includegraphics[width=.4\textwidth]{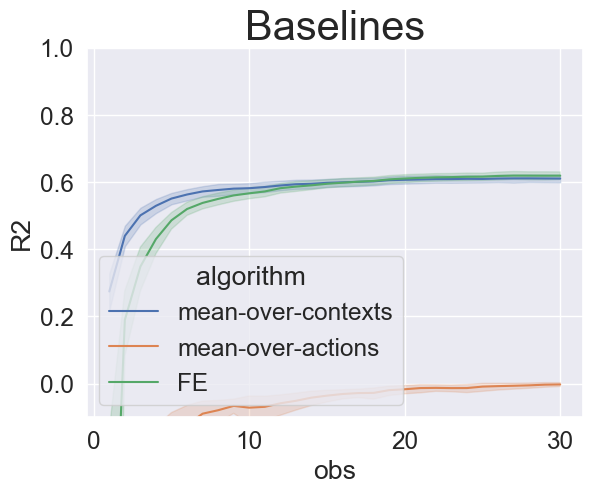}
    \end{subfigure}
    \begin{subfigure}{}
        \centering
        \includegraphics[width=.4\textwidth]{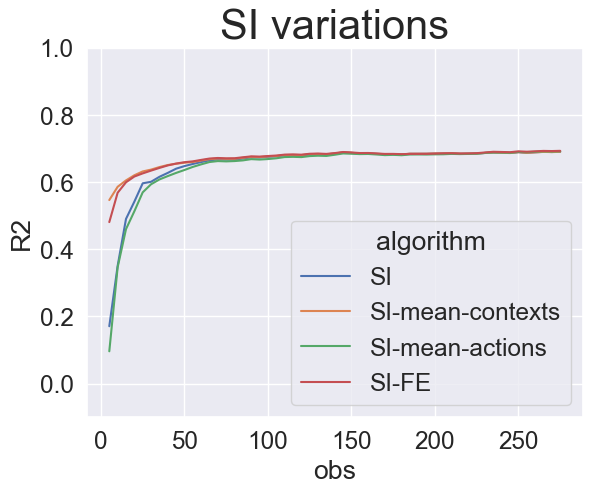}
    \end{subfigure}
    \begin{subfigure}{}
        \centering
        \includegraphics[width=.4\textwidth]{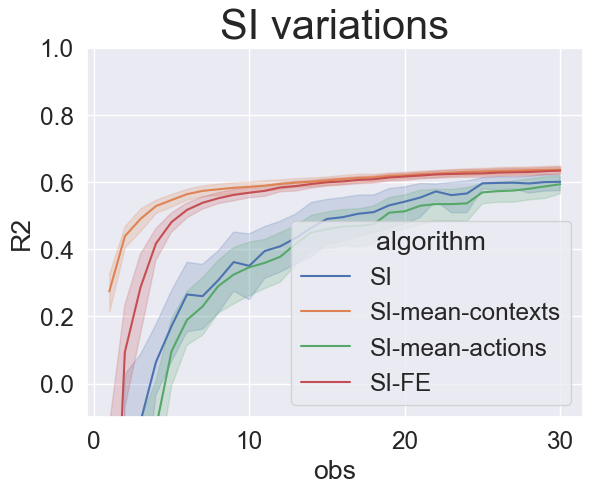}
    \end{subfigure}
    \begin{subfigure}{}
        \centering
        \includegraphics[width=.4\textwidth]{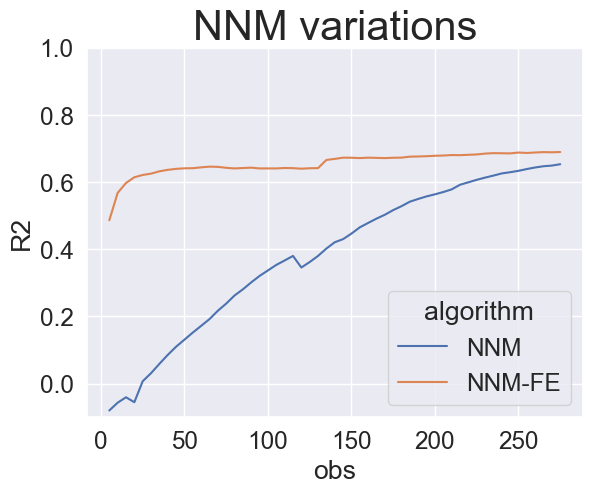}
    \end{subfigure}
    \begin{subfigure}{}
        \centering
        \includegraphics[width=.4\textwidth]{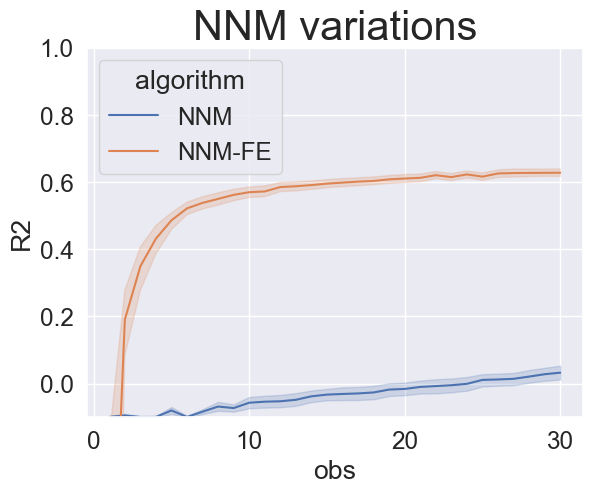}
    \end{subfigure}
    \begin{subfigure}{}
        \centering
        \includegraphics[width=.4\textwidth]{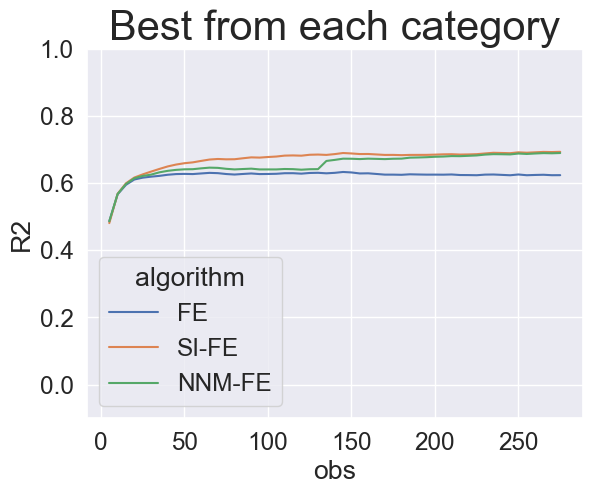}
    \end{subfigure}
    \begin{subfigure}{}
        \centering
        \includegraphics[width=.4\textwidth]{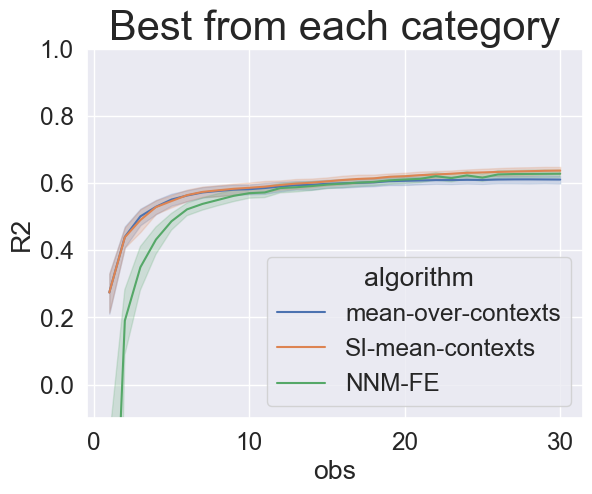}
    \end{subfigure}
    \vspace*{-5mm}
    \caption{\textbf{Dependence on number of observations.}
    "obs" denotes the number of rows and columns observed for each missing entry. Figures from the right-hand side are focused on the low-data regime. We can see some spikes in the curves corresponding to NNM, these are produced by a change in the regularization parameter (e.g. from $10^{-3}$ to $10^{-4}$).}
    \label{fig:increasingobs}
\end{figure}

\clearpage
\subsection{Performance of all the matrix completion algorithms considered} \label{appendix:boxplotsall}

In \rref{sec:performance}, we showed the performance of some matrix completion approaches. Here, in \rref{fig:boxplots_all}, we include more approaches and some variations of the previous ones.

In particular, we see that the nearest neighbors approaches are not significantly better than the baselines. Doubly Robust NN is the method proposed in \cite{dwivedi2022doubly}. CF10 stands for a 10-nearest neighbor Collaborative Filtering approach.
Finally, it is worth noting that SI and its variations outperform the other approaches.

\begin{figure}[h!]
    \centering
    \includegraphics[width = \textwidth]{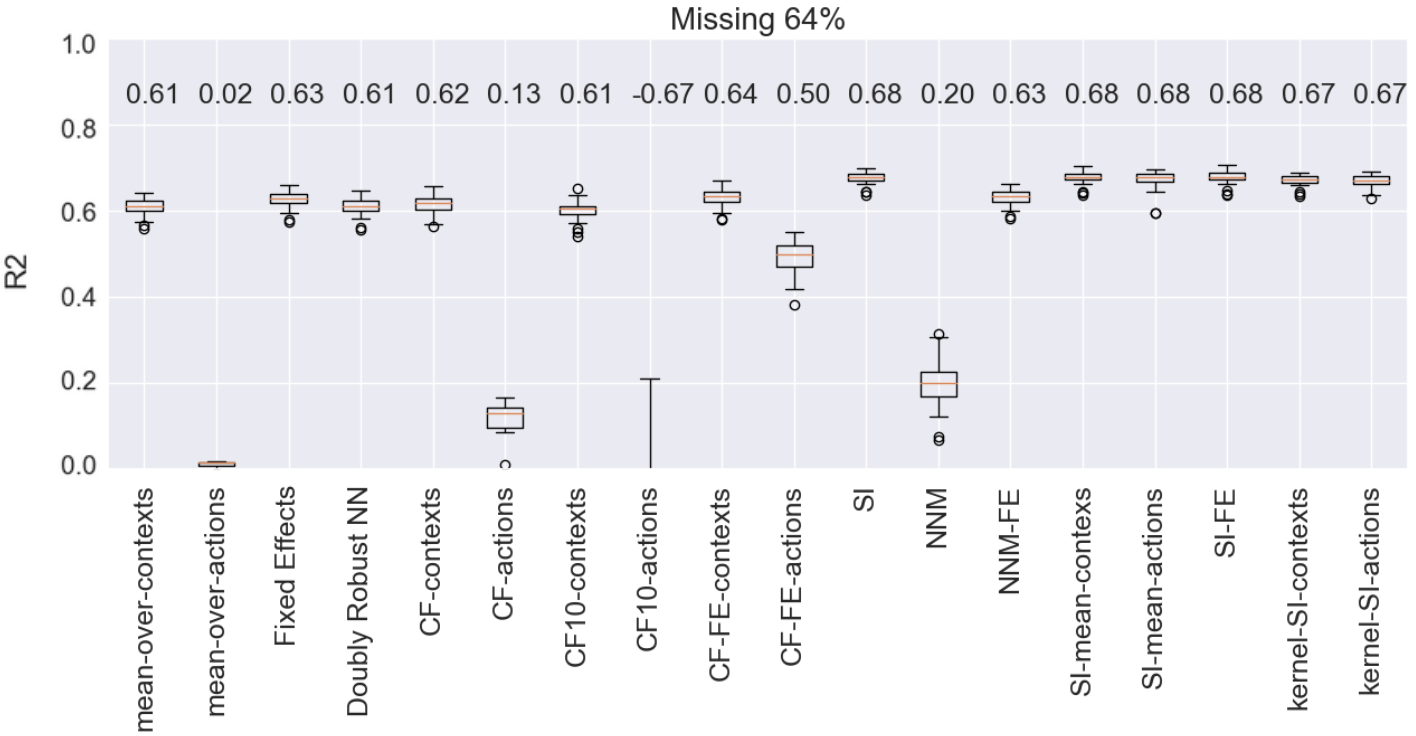}
    \caption{\textbf{All the matrix completion approaches considered.} Here we show the performance of all the algorithms considered, using the missing pattern from Figure \ref{fig:missingsquares}.}
    \label{fig:boxplots_all}
\end{figure}
\subsection{Removal of "killer drugs"} \label{appendix:killerlung}

In this section, we run the same experiments on certain portions of the whole matrix $\bY$.
Looking back at our original matrix in \rref{fig:viability-and-cell-diagram}, we observe that there are some blue (negative) rows that seem to be constant along all the columns. This would mean that this particular drug is killing all the cell lines.

In light of this effect, we call a drug a "killer drug" if it kills more than a certain percentage of cell lines. It is not clear which threshold we should use. In \rref{fig:threshold}, we show the number of "killer drugs" depending on the threshold we define.

For the following results, we used a threshold of 80\%. The resulting matrix can be seen in \rref{fig:killerY}. Note that we have removed almost all the blue lines that we mentioned before. Now the fixed effect along rows does not seem that clear. In particular, the MSE of the baseline decreases significantly compared to the previous setup. This means that the task of completing the task is \textit{easier} for the baseline model. In consequence, getting high $R^2$ values is going to be much harder.

In \rref{fig:thresfoldr2killer} we repeat the same experiment from \rref{fig:svalues} for this new data. As expected, now the $R^2$ value we obtain are much lower. Finally, we show the performance of the different matrix completion approaches on this matrix in \rref{fig:boxplotskiller}.

\begin{figure}[ht]
    \centering
    \includegraphics[width=.4\linewidth]{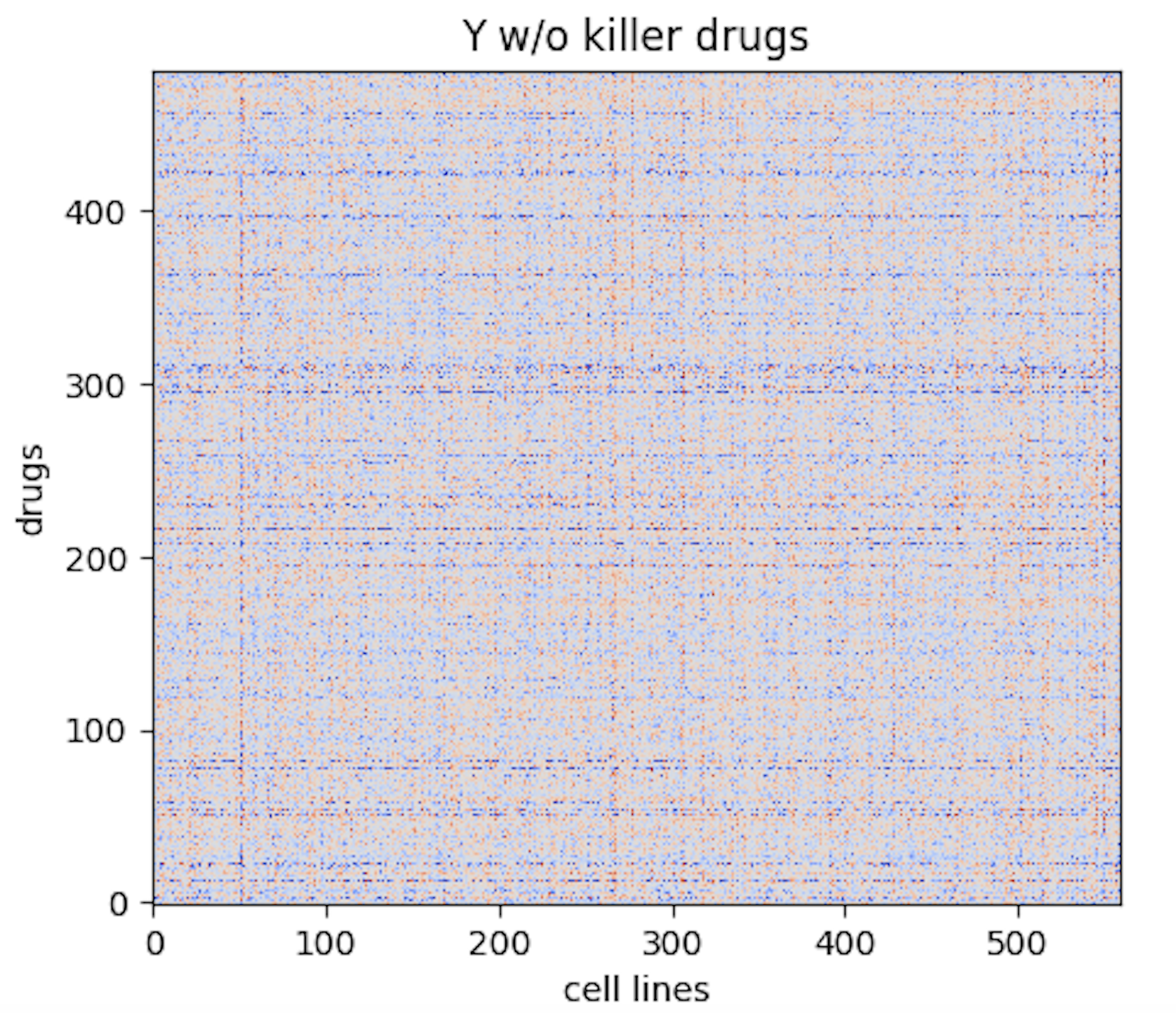}
    \caption{\textbf{Matrix obtained when we remove "killer drugs"}. Here, a "killer drug" is a drug that kills more than 80\% of cell lines.}
    \label{fig:killerY}
\end{figure}

\begin{figure}[ht]
    \centering
    \includegraphics[width=.4\linewidth]{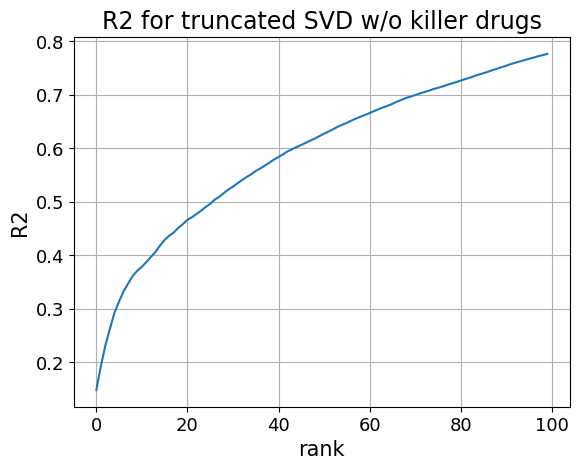}
    \caption{\textbf{Getting high R2 values is even more difficult than before} (compare it with \rref{fig:svalues}). After removing the "killer drugs", MSE(average) decreases from 0.87 to 0.28. This means that the completion task is easier for this baseline, so it is more difficult to obtain high $R^2$ values.}
    \label{fig:thresfoldr2killer}
\end{figure}

\begin{figure}[ht]
    \centering
    \includegraphics[width=\linewidth]{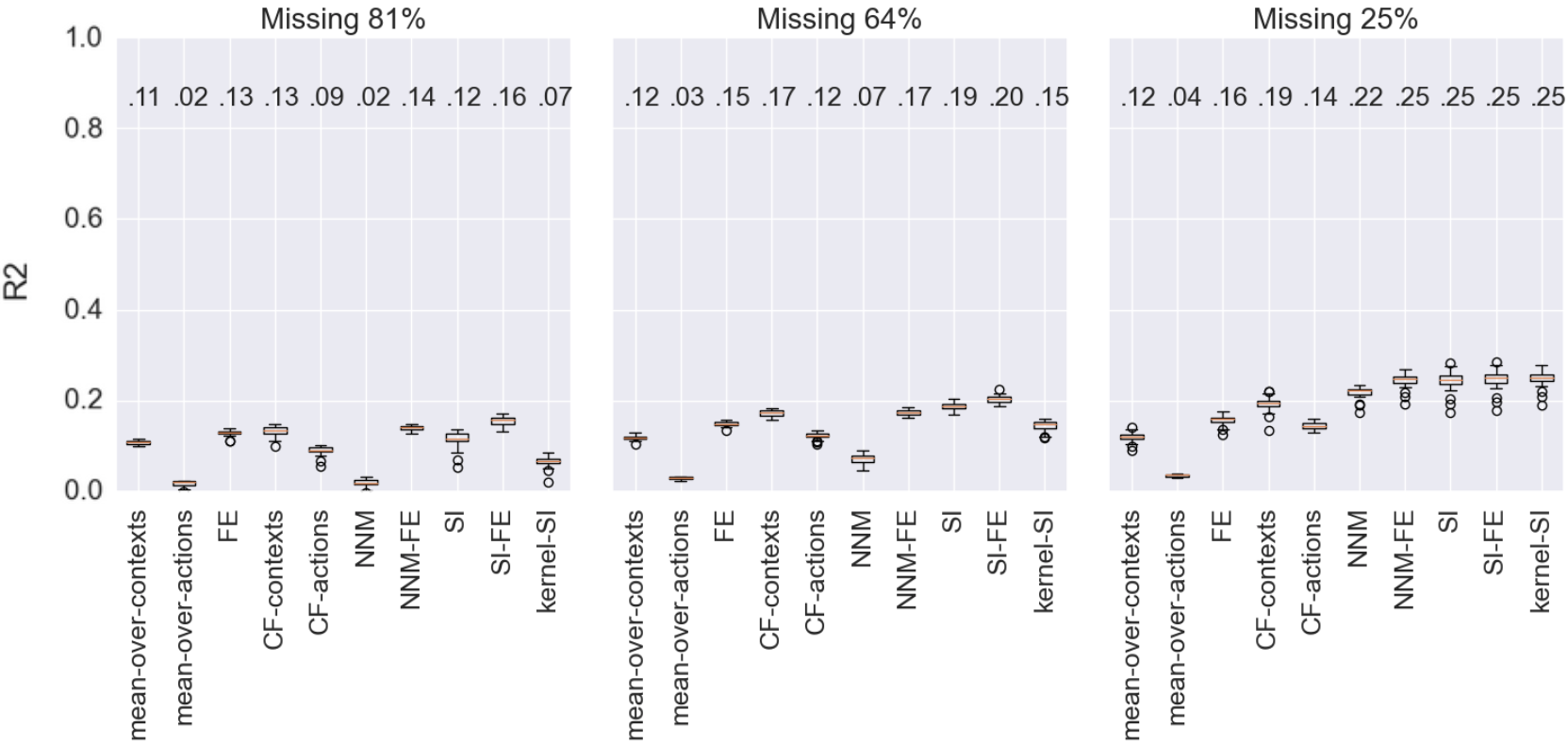}
    \caption{\textbf{Performance after removing "killer drugs".} The gap between mean-over-actions and SI is similar from the one in \rref{fig:boxplots}, despite having more room for improvement. }
    \label{fig:boxplotskiller}
\end{figure}


\begin{figure}[ht]
    \centering
    \includegraphics[width=.5\linewidth]{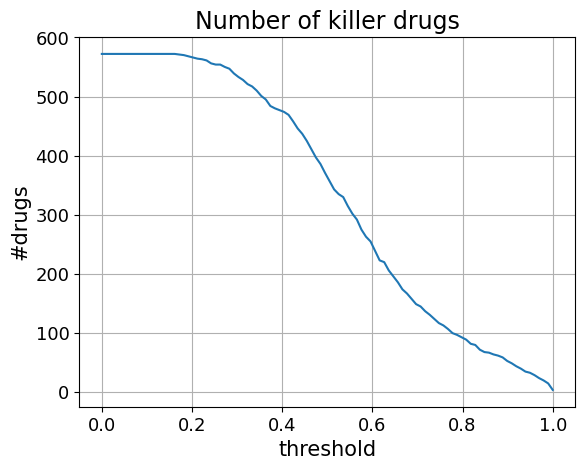}
    \caption{\textbf{Definition of a "killer drug"}. Here we show the number of drugs that kill more than \textit{threshold} $\%$ of cell lines. We consider that a cell line is dead if the viability score is negative. In our experiments, we used a threshold of $80\%$.}
    \label{fig:threshold}
\end{figure}

\end{document}